\pdfoutput=1

\documentclass{article}
\PassOptionsToPackage{numbers, compress}{natbib}
\usepackage[final]{nips_2017}
\usepackage[utf8]{inputenc} 
\usepackage[T1]{fontenc}    
\usepackage{hyperref}       
\usepackage{url}            
\usepackage{booktabs}       
\usepackage{amsfonts}       
\usepackage{nicefrac}       
\usepackage{microtype}      
\usepackage{graphicx}       
\usepackage{subfig}
\usepackage{wrapfig}
\usepackage{amsmath}
\usepackage{amsfonts}
\usepackage{amssymb}
\usepackage{amsthm}
\usepackage{algcompatible}
\usepackage{algorithm}
\usepackage{etoolbox}

\graphicspath{{./fig/}}

\title{Successor Features for \\ Transfer in Reinforcement Learning}

\author{ 
  {\bf Andr\'e Barreto}, \hspace{0.5mm}
  {\bf Will Dabney}, \hspace{0.5mm}
  {\bf R\'{e}mi Munos}, \hspace{0.5mm}
  {\bf Jonathan J. Hunt}, \hspace{0.5mm} \\
  {\bf Tom Schaul}, \hspace{0.5mm}
  {\bf Hado van Hasselt}, \hspace{0.5mm}
  {\bf David Silver} \vspace{1mm} \\ 
   \texttt{\small \{andrebarreto,wdabney,munos,jjhunt,schaul,hado,davidsilver\}@google.com} \vspace{1.5mm} \\
  DeepMind 
}

\newcommand{\cp}{\citep}
\newcommand{\ca}{\citeauthor}
\newcommand{\cy}{\citeyear}
\newcommand{\ct}{\citet}
\newcommand{\ctp}[1]{\ca{#1}'s~\cite{#1}}
\newcommand{\cwp}[1]{\ca{#1}, \cite{#1}}

\newcommand{\mat}[1]{\ensuremath{\boldsymbol{\mathrm{#1}}}}

\newtheorem{theorem}{Theorem}
\newtheorem{lemma}{Lemma}

\newcommand{\R}{\ensuremath{\mathbb{R}}}
\newcommand{\M}{\ensuremath{\mathcal{M}^{\phi}}}
\newcommand{\MM}{\ensuremath{\mathcal{M}}} 
\newcommand{\E}{\ensuremath{\mathrm{E}}}

\newcommand{\Normal}{\ensuremath{\mathcal{N}}}

\newcommand{\w}{\mat{w}}
\newcommand{\z}{\mat{z}}
\newcommand{\Z}{\mat{Z}}
\newcommand{\vphi}{\mat{\phi}}
\newcommand{\vpsi}{\mat{\psi}}
\newcommand{\vprm}{\mat{\z}}
\newcommand{\tvphi}{\ensuremath{\tilde{\mat{\phi}}}}
\newcommand{\tvpsi}{\ensuremath{\tilde{\mat{\psi}} }}
\newcommand{\tmpsi}{\ensuremath{\tilde{\mat{\Psi}} }}

\renewcommand{\t}{\ensuremath{\top}}

\renewcommand{\th}[1]{\text{#1$^\mathrm{th}$}}
\newcommand{\ith}{\th{i}}

\newcommand{\kth}{\th{k}}

\newcommand{\argmin}{\ensuremath{\mathrm{argmin}}}
\newcommand{\argmax}{\ensuremath{\mathrm{argmax}}}
\renewcommand{\min}{\ensuremath{\mathrm{min}}}

\newcommand{\Qmax}{\ensuremath{Q_{\mathrm{max}}}}

\newcommand{\pihmax}{\ensuremath{\pi}}

\newcommand{\rdif}{\ensuremath{\delta}}

\newcommand{\Qt}{\ensuremath{\tilde{Q}}}

\newcommand{\Qtmax}{\ensuremath{\tilde{Q}_{\mathrm{max}}}}
\newcommand{\pitmax}{\ensuremath{{\pi}}}
\newcommand{\Qpitmax}{\ensuremath{Q^{\pi}}}
\newcommand{\Tpitmax}{\ensuremath{T^{\pi}}}
\newcommand{\wt}{\ensuremath{\tilde{\w}}}
\renewcommand{\H}{\mat{H}}

\newcommand{\piexpj}{\ensuremath{_{\pi_j^*}}}
\newcommand{\piexpi}{\ensuremath{_{\pi_i^*}}}
\newcommand{\piexpk}{\ensuremath{_{\pi_k^*}}}

\newcommand{\la}{\ensuremath{\leftarrow}}

\newcommand{\phimax}{\ensuremath{\vphi_{\max}}}

\renewcommand{\S}{\ensuremath{\mathcal{S}}}
\newcommand{\A}{\ensuremath{\mathcal{A}}}
\newcommand{\W}{\ensuremath{\mathcal{W}}}
\newcommand{\T}{\ensuremath{\mathcal{T}}}

\newcommand{\score}{\ensuremath{\mathrm{score}}}
\newcommand{\currscore}{\ensuremath{\mathrm{curr\_score}}}
\newcommand{\ntpu}{\ensuremath{\mathrm{used}}}
\newcommand{\act}{\ensuremath{\mathrm{c}}}

\newcommand{\I}[1]{\ensuremath{\delta\{#1\}}} 
\newcommand{\II}{\ensuremath{\delta}} 

\newcommand{\vf}{\ensuremath{\mat{\varphi}}}
\newcommand{\vfi}{\ensuremath{\mat{\varphi}_i}}
\newcommand{\vfp}{\ensuremath{\mat{\varphi}_p}}
\newcommand{\vweights}{\mat{z}}
\newcommand{\od}{\ensuremath{\mat{o}}}

\begin{document}

\maketitle

\begin{abstract}
Transfer in reinforcement learning 
refers to the notion that generalization should occur 
not only within a task but also across tasks.
We propose a transfer framework for 
the scenario where the reward function
changes between tasks but 
the environment's dynamics remain the same.
Our approach rests on two key ideas: 
\emph{successor features}, a value function representation that decouples 
the dynamics of the environment from the rewards, and 
\emph{generalized policy improvement}, a 
generalization of dynamic programming's 
policy improvement operation that considers 
a set of policies rather than a single one.
Put together, the two ideas lead to 
an approach that integrates seamlessly within the 
reinforcement learning framework and allows the free
exchange of information across tasks. 
The proposed method also provides 
performance guarantees for
the transferred policy even before any learning has taken place.
We derive two theorems that set our approach in firm 
theoretical ground and present experiments that show that
it successfully promotes transfer in practice, significantly
outperforming alternative methods
in a sequence of navigation tasks 
and in the control of a simulated robotic arm.
\end{abstract}

\section{Introduction}
\label{sec:introduction}

Reinforcement learning (RL) provides a framework 
for the development of situated agents 
that learn how to behave while interacting with 
the environment~\cp{sutton98reinforcement}. 
The basic RL loop is defined in an abstract way so as to
capture only the essential aspects of this interaction:
an agent receives observations and selects actions 
to maximize a reward signal.
This setup is generic enough to describe tasks of different levels 
of complexity that may unroll at distinct time scales.
For example, in the task of driving a car, an action 
can be to turn the wheel, make a right turn, or 
drive to a given location.

Clearly, from the point of view of the designer, it is desirable 
to describe a task at the highest level of abstraction possible. 
However, by doing so one may overlook behavioral patterns 
and inadvertently make the task more difficult than it really is.
The task of driving to a location clearly encompasses the subtask of making 
a right turn, which in turn encompasses the action of turning the wheel.
In learning how to drive an agent should be able to identify and exploit 
such interdependencies. More generally, the agent should be able to break 
a task into smaller subtasks and use knowledge 
accumulated in any subset of those to speed up learning in related tasks.
This process of leveraging knowledge acquired in one task to improve 
performance 
on other tasks is called 
\emph{transfer}~\cp{taylor2009transfer,lazaric2012transfer}.

In this paper we look at one specific type of transfer, namely,
when subtasks correspond to different reward functions defined 
in the same environment.
This setup is flexible enough to allow transfer to happen at different levels.
In particular, by appropriately defining the rewards one can
induce different task decompositions. For instance, 
the type of hierarchical decomposition involved in the driving example above
can be induced by changing the frequency 
at which rewards are delivered: a positive 
reinforcement can be given after each maneuver that is well executed 
or only at the final destination.
Obviously, one can also decompose a task into 
subtasks that are independent 
of each other or whose dependency is strictly temporal (that is,
when tasks must be executed in a certain order but no single 
task is clearly ``contained'' within another).

The types of task decomposition discussed above 
potentially allow the agent to tackle 
more complex problems than would be possible were
the tasks modeled as a single monolithic challenge. 
However, in order to apply this divide-and-conquer strategy 
to its full extent the agent 
should have an explicit mechanism to promote transfer
between tasks.
Ideally, we want a transfer approach to have two important properties.
First, the flow of information between tasks should not be dictated by a 
rigid diagram that reflects the relationship between the tasks themselves,
such as hierarchical or temporal dependencies. On the contrary, 
information should be exchanged across tasks whenever useful. 
Second, rather than being posed as a separate problem, transfer should 
be integrated into the RL framework as much as possible, preferably
in a way that is almost transparent to the agent.

In this paper we propose an approach 
for transfer that has the two properties above.
Our method builds on two conceptual pillars that complement each other.
The first is a generalization of \ctp{dayan93improving}
\emph{successor representation}. As the name suggests, in this 
representation scheme each state is described by 
a prediction about the future occurrence of all 
states under a fixed policy.
We present a generalization of \ca{dayan93improving}'s idea
which extends the original scheme to continuous spaces and also 
facilitates the use of approximation. 
We call the resulting scheme \emph{successor features}.
As will be shown, successor features lead to a 
representation of the value function that naturally decouples the dynamics 
of the environment from the rewards, which makes them particularly suitable 
for transfer. 

The second pillar of our framework 
is a generalization of 
\ctp{bellman57dynamic} classic 
policy improvement theorem that extends the original result from one to 
multiple decision policies. 
This novel result shows how knowledge about a set of tasks can be 
transferred to a new task in a way that is completely integrated within RL. 
It also provides performance guarantees on the new task 
\emph{before} any learning has taken place, which opens up the possibility of 
constructing a library of ``skills'' that can be reused 
to solve previously unseen tasks.
In addition, we present a theorem 
that formalizes the notion that an agent should be able to perform well on 
a task if it has seen a similar task before---something clearly 
desirable in the context of transfer. 
Combined, the two results above 
not only set our approach in firm 
ground but also outline the mechanics 
of how to actually implement transfer.
We build on this knowledge to 
propose a concrete method and 
evaluate it in two environments, one encompassing 
a sequence of navigation tasks 
and the other involving the control of a simulated two-joint 
robotic arm.

\section{Background and problem formulation}
\label{sec:background}

As usual, we assume that the interaction between agent and environment
can be modeled as a \emph{Markov decision process} (MDP, 
\cwp{puterman94markov}). 
An MDP is defined as a tuple $M \equiv (\S,\A,p,R,\gamma)$. 
The sets $\S$ and 
$\A$ are
the state and action spaces, respectively;
here we assume that $\S$ and $\A$ are finite whenever such an assumption 
facilitates the presentation, but most of the ideas readily extend
to continuous spaces.
For each $s \in \S$ and $a \in \A$ the function 
 $p(\cdot|s,a)$ gives the next-state 
distribution upon taking action $a$ in state $s$. 
We will often refer to $p(\cdot|s,a)$ as the \emph{dynamics} of the MDP.
The reward received at transition $s \xrightarrow{a}
s'$ is given by the random variable 
$R(s,a,s')$; usually one is interested in
the expected value of this variable, which we will 
denote by $r(s,a,s')$ or 
by $r(s,a) = \E_{S' \sim p(\cdot|s,a)}[r(s, a, S')]$.
The discount factor $\gamma \in [0,1)$ gives 
smaller weights to future rewards. 

The objective of the agent in RL is to find a policy $\pi$---a
mapping from states to actions---that maximizes the expected 
discounted sum of rewards, also called the \emph{return}
$
G_{t} = \sum_{i=0}^{\infty} \gamma^{i} R_{t+i+1},
$
where $R_t = R(S_t, A_t, S_{t+1})$.
One way to address this problem is to use methods
derived from \emph{dynamic programming} (DP), which heavily rely on  the concept
of a {\em value function}~\cp{puterman94markov}.
The \emph{action-value function} of a policy $\pi$
is defined as 
\begin{equation}
\label{eq:Q}
Q^{\pi}(s,a) \equiv \E^{\pi} \left[ G_{t} \,|\, S_{t} = s, A_{t} = a \right],
\end{equation}
where $\E^{\pi}[\cdot]$ denotes expected value when following policy $\pi$.
Once the action-value function of a particular policy
$\pi$ is known, we can derive a new policy $\pi'$ which is 
\emph{greedy} with respect to $Q^{\pi}(s,a)$,
that is, 
$\pi'(s) \in \argmax_{a} Q^{\pi}(s,a)$.
Policy $\pi'$ is guaranteed to be at least as good as (if not
better than) policy $\pi$.
The computation of  $Q^{\pi}(s,a)$ and 
$\pi'$, called {\em policy evaluation} and {\em policy
improvement}, 
define the basic mechanics of RL
algorithms based on DP;
under certain conditions their successive application 
leads to an optimal policy
$\pi^{*}$ that maximizes the expected return 
from every $s \in \S$~\cite{sutton98reinforcement}.

In this paper we are 
interested in the problem of \emph{transfer},
which we define as follows.
Let $\T, \T'$ be two sets of tasks such that $\T' \subset \T$, 
and let $t$ be any task. 
Then there is \emph{transfer} if, 
after training on $\T$, the agent always performs as
well or better on task $t$ than if only trained on $\T'$.
Note that $\T'$ can be the empty set.
In this paper a task will be defined as a specific instantiation of the 
reward function $R(s,a,s')$ for a given MDP.
In Section~\ref{sec:transfer} we will revisit this definition and make it 
more formal.

\section{Successor features}
\label{sec:sf}

In this section we present the concept that will serve as a 
cornerstone for the rest of the paper. We start by presenting a simple reward 
model and then show how it naturally leads to a generalization of 
\ctp{dayan93improving} successor representation (SR).

Suppose that the expected one-step reward associated with transition
$(s, a, s')$ can be computed as
\begin{equation}
\label{eq:reward}
r(s,a, s') = \vphi(s,a, s')^\t\w,
\end{equation}
where $\vphi(s,a, s') \in \R^{d}$ are features of $(s,a,s')$ and 
$\w \in \R^{d}$ are weights. 
Expression~(\ref{eq:reward}) is not restrictive 
because we are not making any assumptions about
$\vphi(s,a,s')$: if we have $\phi_i(s,a,s') = r(s,a,s')$ for some $i$, for 
example, we
can clearly recover any reward function exactly.
To simplify the notation, let $\vphi_t = \vphi(s_t, a_t, s_{t+1})$. Then, 
by simply rewriting the definition of the action-value function in~(\ref{eq:Q}) 
we have
 \begin{align}
  \label{eq:sf} 
  \nonumber Q^{\pi}(s,a) 
   & = \E^{\pi}\left[ r_{t+1} + \gamma r_{t+2} +  ... \,|\, S_t = s, A_t = a 
\right]  \\
\nonumber & = \E^{\pi}\left[ \vphi_{t+1}^{\t} \w + \gamma \vphi_{t+2}^{\t} \w 
+ ... \, | \, S_t = s, A_t = a \right] \\
  & = \E^{\pi} \left[{\textstyle \sum_{i=t}^{\infty}} \gamma^{i-t} 
\vphi_{i+1} \,|\, S_t = s, A_t = a \right]^{\t} \w 
  = \vpsi^{\pi}(s,a)^{\t} \w.
 \end{align}
The decomposition~(\ref{eq:sf}) has appeared before
in the literature under different 
names and interpretations, as discussed in 
Section~\ref{sec:related_work}. 
Since here we propose to look at~(\ref{eq:sf})
as an extension of \ctp{dayan93improving} SR,
we call $\vpsi^{\pi}(s,a)$ the 
\emph{successor features} (SFs) of $(s,a)$ under policy $\pi$. 

The \ith\ component of $\vpsi^\pi(s,a)$ gives the 
expected discounted sum of
$\phi_i$ when following policy $\pi$ starting from 
$(s,a)$. In the particular case where $\S$ and $\A$ are finite and 
$\vphi$ is a tabular representation 
of $\S \times \A \times \S$---that is, $\vphi(s,a,s')$ is a one-hot vector 
in $\R^{|\S|^2|\A|}$---$\vpsi^\pi(s,a)$ is the discounted sum of 
occurrences,  under $\pi$, of each possible transition. 
This is essentially the concept of SR
extended from the space $\S$ to the set $\S \times \A \times 
\S$~\cp{dayan93improving}.

One of the contributions of this paper is precisely to generalize 
SR to be used with function approximation, 
but the exercise of deriving the concept as above provides 
insights already in the tabular case. 
To see this, note that in the tabular case the entries of $\w \in 
\R^{|\S|^2|\A|}$
are the function $r(s,a,s')$ and 
suppose that $r(s,a,s') \ne 0$ in only a small 
subset $\W \subset \S \times \A \times \S$. 
From~(\ref{eq:reward}) and~(\ref{eq:sf}), 
it is clear that 
the cardinality of $\W$, and not of $\S \times \A \times \S$, is what 
effectively 
defines the dimension of the representation $\vpsi^{\pi}$, since there 
is no point in having $d > |\W|$. Although this fact is hinted at 
by \ct{dayan93improving}, it becomes more apparent 
when we look at SR as a particular case of SFs.

SFs extend SR in two other ways. 
First, the concept readily applies to continuous state and 
action 
spaces
without any modification. Second, 
by explicitly casting~(\ref{eq:reward}) and~(\ref{eq:sf}) as inner products
involving feature vectors, 
SFs make it evident how to incorporate 
function approximation: as will be shown, these vectors can
be learned from data.

The representation in~(\ref{eq:sf}) requires two components to be learned,
\w\ and $\vpsi^{\pi}$. Since the latter is the expected discounted sum of 
$\vphi$ under $\pi$, we must either be given $\vphi$ or learn it as well.
Note that approximating $r(s,a,s') \approx \vphi(s,a,s')^\t\wt$ 
is a supervised learning problem, 
so we can use 
well-understood techniques from the field 
to learn \wt\ (and potentially \tvphi, 
too)~\cp{hastie2002elements}. 
As for $\vpsi^{\pi}$, we note that 
 \begin{align}
 \label{eq:bellman_psi}
  \vpsi^\pi(s, a) 
  & = \vphi_{t+1}+ \gamma E^{\pi}[\vpsi^\pi(S_{t+1}, \pi(S_{t+1}))  \,|\, S_t = s, A_t = a ],
\end{align}
that is, SFs satisfy a Bellman equation in which 
$\phi_i$ play the role of rewards---something also 
noted by \ct{dayan93improving}
regarding SR. Therefore, in principle \emph{any} 
RL method can be used to compute 
$\vpsi^{\pi}$~\cp{szepesvari2010algorithms}.

The SFs $\vpsi^{\pi}$ summarize the dynamics 
induced by $\pi$ in a given environment.
As shown in~(\ref{eq:sf}), this allows for a 
modular representation of $Q^{\pi}$ in which 
the MDP's dynamics are decoupled from its rewards, which are 
captured by the weights $\w$.
One potential benefit of having such a decoupled representation is that 
only the relevant module must be relearned when either the dynamics or the 
reward changes, which may serve as an argument in favor of adopting SFs 
as a general approximation scheme for RL. 
However, in this paper we focus on a scenario where the decoupled
value-function approximation provided by SFs is exploited to its full extent, 
as we discuss next.

\section{Transfer via successor features}
\label{sec:transfer}

We now return to the discussion about transfer in RL.
As described, we are interested in the scenario where 
all components of an MDP are fixed, except for the reward function.
One way to formalize this model is through~(\ref{eq:reward}):
if we suppose that $\vphi \in \R^{d}$ is fixed, any 
$\w \in \R^d$ gives rise to a new MDP. 
Based on this observation, we define 
\begin{equation}
\label{eq:M}
\M(\S, \A, p, \gamma) \hspace{-1mm} \equiv \{ M(\S, \A, p, r, \gamma) \; | \;
r(s,a,s') \hspace{-1mm} = \vphi(s,a,s')^\t\w \},
\end{equation}
that is, \M\ is the set of MDPs induced by $\vphi$ through 
all possible instantiations of \w.
Since what differentiates the MDPs in \M\ is essentially the agent's goal,
we will refer to $M_i \in \M$ as 
a \emph{task}. 
The assumption is that we are interested in solving (a subset of) the 
tasks in the environment \M. 

Definition~(\ref{eq:M}) 
is a natural way of modeling some scenarios of interest.
Think, for example, how the desirability of water or
food changes depending on whether an animal is thirsty or hungry. 
One way to model this type of preference shifting is to suppose 
that the vector \w\ appearing in~(\ref{eq:reward}) reflects 
the taste of the agent at any given point in time~\cp{natarajan2005dynamic}. 
Further in the paper we will present experiments that reflect this scenario.
For another illustrative example, imagine that the agent's goal is to produce 
and sell a combination of goods whose production line is relatively stable
but whose prices vary considerably over time. In this case updating the price
of the products corresponds to picking a new \w. 
A slightly different way of motivating~(\ref{eq:M}) is to suppose that the 
environment itself is changing, that is, the element $\w_i$ 
indicates not only desirability, but also 
availability, of feature $\phi_i$.

In the examples above it is desirable for the agent to build 
on previous 
experience to improve its performance on a new setup. More concretely, 
if the agent knows good policies for the set of tasks
$\MM \equiv \{M_1, M_2, ..., M_n\}$, with $M_i \in \M$, it should be able to 
leverage    
this knowledge to improve its behavior on a new
task $M_{n+1}$---that is, it should perform better than it would had it been 
exposed to only a subset of the original tasks, $\MM' \subset \MM$.
We can assess the performance of an agent on task $M_{n+1}$
based on the value function of the policy followed 
after $\w_{n+1}$ has become available but 
\emph{before} any policy improvement has taken place 
in $M_{n+1}$.\footnote{Of course $\w_{n+1}$ can, and will be, learned, 
as discussed in Section~\ref{sec:gpi_sf} and illustrated in 
Section~\ref{sec:experiments}. Here we assume that $\w_{n+1}$
is given to make our performance criterion clear.}
More precisely, suppose that an agent has 
been exposed to each one of the tasks $M_i \in \MM'$.
Based on this experience, and on the new $\w_{n+1}$,
 the agent computes a policy $\pi'$ 
that will define its initial behavior in $M_{n+1}$.
Now, if we repeat the experience replacing $\MM'$ with $\MM$,
the resulting policy $\pi$ should be such that 
$Q^{\pi}(s,a) \ge Q^{\pi'}(s,a)$ for all $(s,a) \in \S \times \A$.

Now that our setup is clear we can start to describe our solution 
for the transfer problem discussed above. We do so in two stages.
First, we present a generalization of DP's notion
of policy improvement whose interest may go beyond the current work.
We then show how SFs can be used to implement
this generalized form of policy improvement in an efficient and elegant way.  

\subsection{Generalized policy improvement}
\label{sec:gpi}

One of the fundamental results in RL is \ctp{bellman57dynamic} 
\textit{policy improvement theorem}. 
In essence, the theorem states 
that acting greedily with respect 
to a policy's value function gives rise to another policy whose performance
is no worse than the former's. This is the driving force behind 
DP, and most RL algorithms that compute a value function are exploiting 
Bellman's result in one way or another. 

In this section we extend the policy improvement theorem to the scenario where 
the new policy is to be computed based on the value functions 
of a \emph{set} of policies.
We show that this extension can be done 
in a natural way, by acting 
greedily with respect to the maximum over 
the value functions available. Our result is 
summarized in the theorem below. 
\begin{theorem}
\label{teo:gpi}
{\bf (Generalized Policy Improvement)} Let $\pi_1$, $\pi_2$, ..., $\pi_n$ be $n$ 
decision policies and 
let $\Qt^{\pi_1}$, $\Qt^{\pi_2}$, ..., $\Qt^{\pi_n}$ be approximations of their 
respective action-value functions
such that 
\begin{equation}
\label{eq:epsilon}
|Q^{\pi_i}(s,a) - \Qt^{\pi_i}(s,a)| \le \epsilon \, \text{ for all } s \in \S, a \in \A, \text{ and }  i \in \{1, 2, ..., n\}.
\end{equation}
Define
\begin{equation}
\label{eq:pitmax}
\pitmax(s) \in \mathop{\argmax}_a \max_i \Qt^{\pi_i}(s,a).
\end{equation}
Then, 
\begin{equation}
\label{eq:Qpitmax}
\Qpitmax(s,a)  \ge \max_i Q^{\pi_i}(s,a) - \dfrac{2}{1 - \gamma} \epsilon
\end{equation}
for any $s \in \S$ and $a \in \A$,
where \Qpitmax\ is the action-value function of \pitmax.
\end{theorem}
The proofs of our theoretical results are in the supplementary material.
As one can see, our theorem covers the case where the policies' value functions 
are not computed exactly, either because function approximation is used 
or because some exact algorithm has not be run to completion. 
This error is captured by $\epsilon$ in~(\ref{eq:epsilon}), which 
re-appears as a penalty term in the lower bound~(\ref{eq:Qpitmax}).
Such a penalty is inherent to the presence of approximation in RL,
and in fact it is identical to the penalty incurred in the single-policy case
(see {\sl e.g.} \ca{bertsekas96neuro-dynamic}'s Proposition~6.1~\cite{bertsekas96neuro-dynamic}).

In order to contextualize generalized policy improvement (GPI) 
within the broader scenario 
of DP, suppose for a moment that $\epsilon = 0$.
In this case Theorem~\ref{teo:gpi} states that $\pi$ will perform no worse 
than \emph{all} of the policies 
$\pi_1$, $\pi_2, ..., \pi_n$. 
This is interesting because in general $\max_i Q^{\pi_i}$---the function 
used to induce $\pi$---is not the value function 
of any particular policy.
It is not difficult to see that $\pi$
will be strictly better than all previous policies if 
no single policy dominates all other policies, that is, if
$\argmax_i \max_a \Qt^{\pi_i}(s,a) \cap \argmax_i \max_a \Qt^{\pi_i}(s',a) = 
\emptyset$
for some $s, s' \in \S$.
If one policy does dominate all others, GPI reduces to 
the original policy improvement theorem. 

If we consider the usual DP loop, in which policies 
of increasing performance are computed in sequence,  
our result is not of much use because the
most recent policy will always dominate all others. 
Another way of putting it is to say that after Theorem~\ref{teo:gpi}
is applied once adding the resulting $\pi$ to the set 
$\{\pi_1$, $\pi_2, ..., \pi_n\}$ will reduce the next improvement step 
to standard policy improvement, and thus the policies  
$\pi_1$, $\pi_2, ..., \pi_n$ can be simply discarded.
There are however two situations in which our result may be of interest.
One is when we have many policies $\pi_i$ being evaluated in parallel.
In this case GPI provides a principled strategy for combining 
these policies. 
The other situation in which our result may be useful 
is when the underlying MDP changes, as we discuss next.

\subsection{Generalized policy improvement with successor features}
\label{sec:gpi_sf}

We start this section by extending our notation slightly 
to make it easier to 
refer to the quantities involved in transfer learning. 
Let $M_i$ be a task in \M\ defined by $\w_i \in \R^d$.
We will use $\pi^*_i$ to refer to an optimal policy of MDP $M_i$
and use $Q^{\piexpi}_i$ to refer to its value function.
The value function of $\pi^*_i$ when executed in $M_j \in \M$
will be denoted by $Q^{\piexpi}_j$. 

Suppose now that an agent has computed optimal policies for the tasks 
$M_1, M_2, ..., M_n \in \M$. 
Suppose further that when presented with a new task $M_{n+1}$ the agent computes
$\{Q^{\pi_1^*}_{n+1}, Q^{\pi_2^*}_{n+1},..., Q^{\pi_n^*}_{n+1}\}$, the evaluation of each $\pi^*_i$ 
under the new reward function induced by $\w_{n+1}$.
In this case, applying the GPI theorem 
to the newly-computed set of value 
functions will give rise to a policy that performs at least as well as 
a policy based on any subset of these, including the empty set. 
Thus, this strategy satisfies our definition of successful transfer.

There is a caveat, though. Why would one waste time computing the value 
functions of 
$\pi_1^*, \pi^*_2$, ..., $\pi_n^*$, whose performance in $M_{n+1}$ may be 
mediocre,
if the same amount of resources can be allocated to compute a sequence of 
$n$ policies with increasing performance? This is where SFs 
come into play.
Suppose that we have learned the functions $Q^{\piexpi}_i$ using 
the representation scheme shown in~(\ref{eq:sf}).
Now, if the reward changes to 
$r_{n+1}(s,a,s') = \vphi(s,a,s')^\t\w_{n+1}$, 
as long as we have $\w_{n+1}$  
we can compute the new value function 
of $\pi_i^*$ by simply making 
$Q^{\piexpi}_{n+1}(s,a) = \vpsi^{\piexpi}(s,a)^\t\w_{n+1}$.
This reduces the computation of all 
${Q}^{\piexpi}_{n+1}$ to the much simpler supervised problem of 
approximating~$\w_{n+1}$. 

Once the functions ${Q}^{\piexpi}_{n+1}$ have been computed, we can apply GPI
to derive a policy $\pi$ whose performance on $M_{n+1}$ is no worse than the 
performance 
of $\pi_1^*, \pi^*_2, ..., \pi^*_n$ on the same task. A question 
that arises in this case is whether we can provide stronger guarantees on the 
performance of 
$\pi$ by exploiting the structure shared by the tasks in \M.
The following theorem answers this question in the affirmative. 

\begin{theorem}
\label{teo:gpi_sf}
Let $M_i \in \M$ and let $Q^{\piexpj}_i$ be the action-value function of an optimal 
policy of $M_j \in \M$ when executed in $M_i$. Given approximations 
$\{\Qt^{\pi_1^{*}}_i, \Qt^{\pi_2^{*}}_i, ..., \Qt^{\pi_n^{*}}_i\}$ such that 
\begin{equation}
\label{eq:ub_epsilon}
 \left|Q^{\piexpj}_i(s, a) - \Qt^{\piexpj}_i(s,a) \right| \le \epsilon
\end{equation}
for all $s \in \S$, $a \in \A$, and $j \in \{1, 2, ..., n\}$,
let 
$
\label{eq:pimax}
\pihmax(s) \in \mathop{\argmax}_a \max_j \Qt^{\piexpj}_i (s,a).
$
Finally, let $\phimax = \max_{s,a} ||\vphi(s,a)||$, where $||\cdot||$ is 
the norm induced by the inner product adopted. Then,
\begin{equation}
\label{eq:bound_w}
Q^{\piexpi}_i(s,a) - Q^{\pi}_i(s,a) 
\le \dfrac{2 }{1- \gamma}  \left(\phimax \, \min_j || \w_i - \w_j|| + \epsilon 
\right).
\end{equation}
\end{theorem}
Note that we used $M_i$ rather than $M_{n+1}$ in the theorem's 
statement to remove any suggestion of order among the tasks. 
Theorem~\ref{teo:gpi_sf} is a specialization of Theorem~\ref{teo:gpi} for the 
case where the set of value functions used to compute $\pi$ 
are associated with tasks in the form 
of~(\ref{eq:M}). As such, it provides stronger guarantees: instead
of comparing the performance of $\pi$ with that of the previously-computed 
policies $\pi_j$, Theorem~\ref{teo:gpi_sf} 
quantifies the loss incurred by following $\pi$ 
as opposed to one of $M_i$'s optimal policies.

As shown in~(\ref{eq:bound_w}), the loss $Q^{\piexpi}_i(s,a) - Q^{\pi}_i(s,a)$
is upper-bounded by two terms.
The term 
$2\phimax \min_j || \w_i - \w_j|| / (1- \gamma)$
is of more interest here because it 
reflects the structure of \M.
This term is a multiple of the distance between 
$\w_i$, the vector describing the task we are currently interested in, and the 
closest $\w_j$ for which we have computed a policy. This formalizes the 
intuition that the agent should perform well in task $\w_i$ if it
has solved a similar task before. 
More generally, the term in question relates the concept of distance
in $\R^d$ with difference in performance in \M. 
Note that this correspondence depends on the specific set of features $\vphi$ 
used, which raises the interesting question of 
how to define $\vphi$ such that tasks that are 
close in $\R^d$ induce policies that are also similar
in some sense. Regardless of how exactly $\vphi$ is defined, 
the bound~(\ref{eq:bound_w}) allows for powerful
extrapolations. For example, by covering the 
relevant subspace of $\R^d$ with balls of 
appropriate radii centered at $\w_j$ 
we can provide performance guarantees for \emph{any} task~$\w$~\cp{mehta2008transfer}.
This corresponds to building a library of \emph{options} (or ``skills'') that 
can be used to solve any task in a (possibly infinite) 
set~\cp{sutton99between}. 
In Section~\ref{sec:experiments} we illustrate this 
concept with experiments.

Although Theorem~\ref{teo:gpi_sf} is inexorably related 
to the characterization of \M\ in~(\ref{eq:M}), it does not depend on 
the definition of SFs in any way. 
Here SFs are 
the \emph{mechanism} used to efficiently apply the protocol suggested by 
Theorem~\ref{teo:gpi_sf}. When SFs are used the 
value function approximations are given by 
$\Qt^{\piexpj}_i(s, a) = \tilde{\vpsi}^{\piexpj}(s,a)^\t \wt_i$.
The modules $\tilde{\vpsi}^{\piexpj}$ 
are computed and stored when the agent is learning the tasks $M_j$;
when faced with a new task $M_i$ the agent computes an approximation
of $\w_i$, which is a supervised learning problem,
and then uses the GPI policy $\pi$ defined in Theorem~\ref{teo:gpi_sf}
to learn $\tilde{\vpsi}^{\piexpi}$.
Note that we do not assume that either 
${\vpsi}^{\piexpj}$ or $\w_i$ is computed exactly:
the effect of errors in $\tilde{\vpsi}^{\piexpj}$ and $\wt_i$
are accounted for 
by the term $\epsilon$ appearing in~(\ref{eq:ub_epsilon}).
As shown in~(\ref{eq:bound_w}), 
if $\epsilon$ is small and the agent has seen enough tasks 
the performance of $\pi$ on $M_i$ should already be good, 
which suggests it may also speed up 
the process of learning $\tilde{\vpsi}^{\piexpi}$.

Interestingly, Theorem~\ref{teo:gpi_sf} also provides guidance 
for some practical algorithmic choices. Since in an actual 
implementation one wants to limit the number of SFs $\tilde{\vpsi}^{\piexpj}$
stored in memory, the corresponding vectors $\wt_j$ can be used to 
decide which ones to keep. For example, one can create a new  
$\tilde{\vpsi}^{\piexpi}$ only when 
$\min_j || \wt_i - \wt_j||$ is above a given threshold;
alternatively, once the maximum number of SFs has been reached, 
one can replace $\tilde{\vpsi}^{\piexpk}$,
where $k = \argmin_j || \wt_i - \wt_j||$ (here $\w_i$ is the 
current task).

\section{Experiments}
\label{sec:experiments}

In this section we present our main experimental results. Additional details, 
along with further results and analysis, can be found 
in Appendix~\ref*{sec:details_experiments} of the supplementary material.

The first environment we consider involves navigation tasks 
defined over a two-dimensional continuous space
composed of four rooms (Figure~\ref{fig:forage_maze}). 
The agent starts in 
one of the rooms and must reach a 
goal region located in the farthest room. 
The environment has objects that can 
be picked up by the agent by passing over them. 
Each object belongs to one of three classes determining the associated reward. 
The objective of the agent is to pick up the ``good'' objects 
and navigate to the goal while avoiding ``bad'' 
objects. 
The rewards associated with object classes change at 
every $20\,000$ transitions, giving rise to very different tasks 
(Figure~\ref{fig:forage_maze}).
The goal is to maximize the sum of rewards accumulated over 
a sequence of $250$ tasks, with each task's rewards sampled uniformly from $[-1, 1]^3$.

\begin{wrapfigure}{r}{0.58\textwidth}
  \hspace{2mm}
  \includegraphics[scale=0.75]{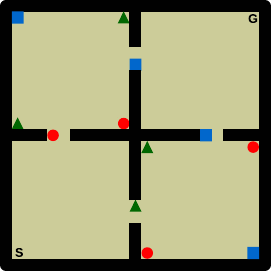}
  \hspace{2mm}
  \includegraphics[scale=0.75]{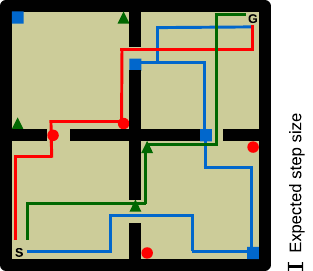}
\caption{Environment layout and some examples of optimal trajectories 
associated with specific tasks.
The shapes of the objects represent their classes; `S' is the start state and 
`G' is the goal. 
\label{fig:forage_maze}
 \vspace{-5mm}
}
\end{wrapfigure}

We defined a straightforward instantiation of our 
approach in which both \wt\ and $\tilde{\vpsi}^{\pi}$ are computed
incrementally in order to minimize 
losses induced by~(\ref{eq:reward}) and~(\ref{eq:bellman_psi}). 
Every time the task changes 
the current $\tilde{\vpsi}^{\pi_i}$ is 
stored and a new $\tilde{\vpsi}^{\pi_{i+1}}$ is created. 
We call this method SFQL as a 
reference to the 
fact that SFs are learned through an algorithm 
analogous to $Q$-learning (QL)---which is used 
as a baseline in our comparisons~\cp{watkins92qlearning}. 
As a more challenging reference point we report results for 
a transfer method called \emph{probabilistic policy 
reuse}~\cp{fernandez2010probabilistic}. We adopt a version of 
the algorithm that builds on QL and reuses all policies 
learned. The resulting method, PRQL, 
is thus directly comparable to SFQL. The details of  QL, PRQL, and SFQL, 
including their pseudo-codes, are given in 
Appendix~\ref*{sec:details_experiments}.

We compared two versions of SFQL. In the first one, 
called SFQL-\vphi, we assume the agent has access to 
features \vphi\ that perfectly predict the 
rewards, as in~(\ref{eq:reward}). 
The second version of our agent had to learn an approximation 
$\tvphi \in \R^h$ directly from data collected by QL 
in the first $20$ tasks. 
Note that $h$ may not coincide with 
the true dimension of $\vphi$, which in this case is $4$;
we refer to the different instances of our algorithm as SFQL-$h$.
The process 
of learning \tvphi\ followed the 
multi-task learning protocol proposed by
\ct{caruana97multitask} and \ct{baxter2000model}, and
described in detail in 
Appendix~\ref*{sec:details_experiments}. 

The results of our experiments can be seen 
in Figure~\ref{fig:results_four_room}. As shown, 
all versions of SFQL significantly outperform the other two methods,
with an improvement on the average return of more than $100\%$ when compared 
to PRQL, which itself improves on QL by around $100\%$. 
Interestingly, SFQL-$h$ seems to achieve good overall performance 
\emph{faster} than SFQL-$\phi$, even though the latter uses 
features that allow for an exact representation of the rewards. One possible 
explanation is that, unlike their counterparts $\phi_i$, 
the features $\tilde{\phi}_i$ are activated over most of the space 
$\S \times \A \times \S$, which results in a dense pseudo-reward signal that facilitates learning.

\begin{figure}[b]
\includegraphics[width=\textwidth,height=50mm]{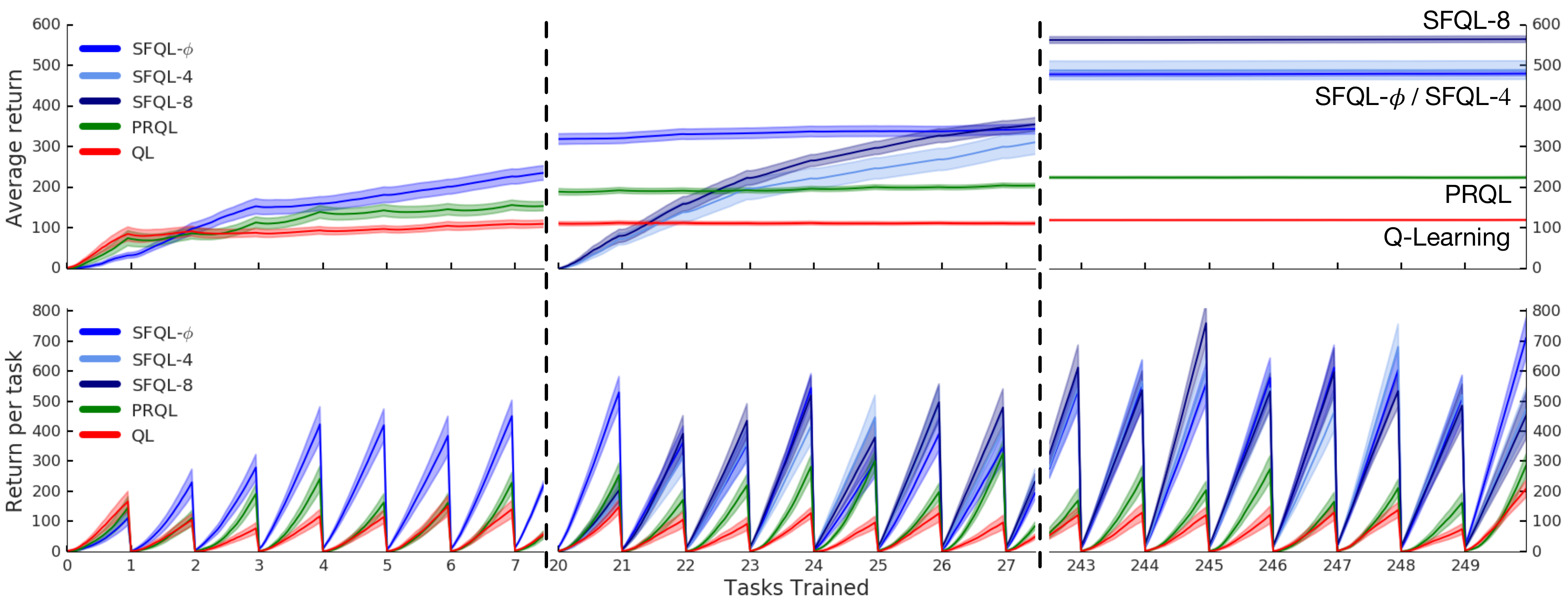} 
\caption{Average and cumulative return per task in the four-room domain.
SFQL-$h$ receives no reward during the first $20$ tasks 
while learning $\tvphi$. 
Error-bands show one standard error over $30$ runs.
\label{fig:results_four_room}
}
\end{figure}

The second environment we consider is a set of control tasks defined in the MuJoCo 
physics engine~\cp{todorov2012mujoco}. 
Each task consists in moving a two-joint torque-controlled 
simulated robotic arm to a specific target location;  
thus, we refer to this environment as ``the reacher domain.''
We defined $12$ tasks, but only allowed 
the agents to train in $4$ of them (Figure~\ref{fig:results_reacher}c). 
This means that the agent must be able to perform well 
on tasks that it has never experienced during training. 

In order to solve this problem, we adopted essentially the same algorithm as 
above, but we 
replaced QL with \ca{mnih2015human}'s DQN---both as a baseline and as the basic 
engine underlying the SF agent~\cp{mnih2015human}. The resulting method, which 
we call SFDQN, is an illustration of how our method can be naturally combined 
with complex nonlinear approximators such as neural 
networks. The features $\phi_i$ used by SFDQN are the 
negation of the distances to the center of the $12$ 
target regions. As usual in experiments of this type,  
we give the agents a description of the 
current task: for DQN the target coordinates are given 
as inputs, while for SFDQN
this is provided as an one-hot 
vector $\w_t \in \R^{12}$~\cp{lillicrap2015continuous}. 
Unlike in the previous experiment, in the current setup each transition 
was used to train all four $\tilde{\vpsi}^{\pi_i}$ through
losses derived from~(\ref{eq:bellman_psi}). Here
$\pi_i$ is the GPI policy on the \ith\ task:
$\pi_i(s) \in \argmax_a \max_{j} \tvpsi_j(s,a)^\t \w_i$.

Results are shown in Figures~\ref{fig:results_reacher}a 
and~\ref{fig:results_reacher}b.
Looking at the training curves, we see 
that whenever a task is selected for training SFDQN's return 
on that task quickly improves 
and saturates at near-optimal performance.
The interesting point to be noted is that, when learning a given task, 
SFDQN's performance 
also improves in all other tasks, including the test ones, for which 
it does not have specialized policies.
This illustrates how the combination of SFs and GPI can give rise 
to flexible agents able to perform 
well in \emph{any} task of a set of tasks with shared 
dynamics---which in turn can be seen as 
both a form of temporal abstraction and a
step towards more general hierarchical RL~\cp{sutton99between,barto2003recent}. 

\begin{figure}[t]
\centering
\begin{subfloat}[
Performance on training tasks 
(faded dotted lines in the background are DQN's results).]{
{
\begin{minipage}{0.5\textwidth}
\label{fig:total_return}
\includegraphics[scale=0.28]{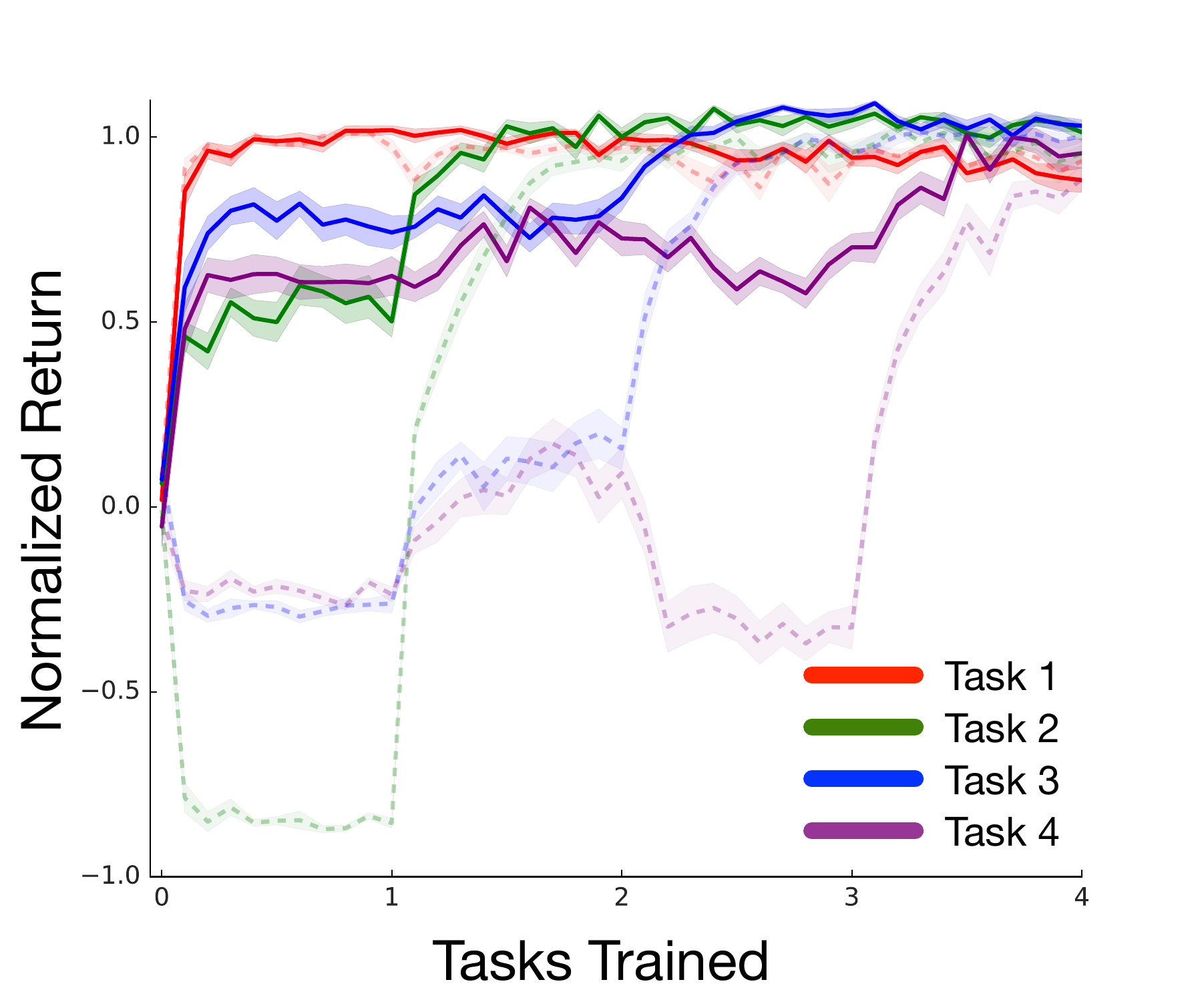} 
\end{minipage}
}
\hspace{5mm}
}
\end{subfloat}
\begin{subfloat}
{
\begin{tabular}{c}
 \includegraphics[scale=0.15]{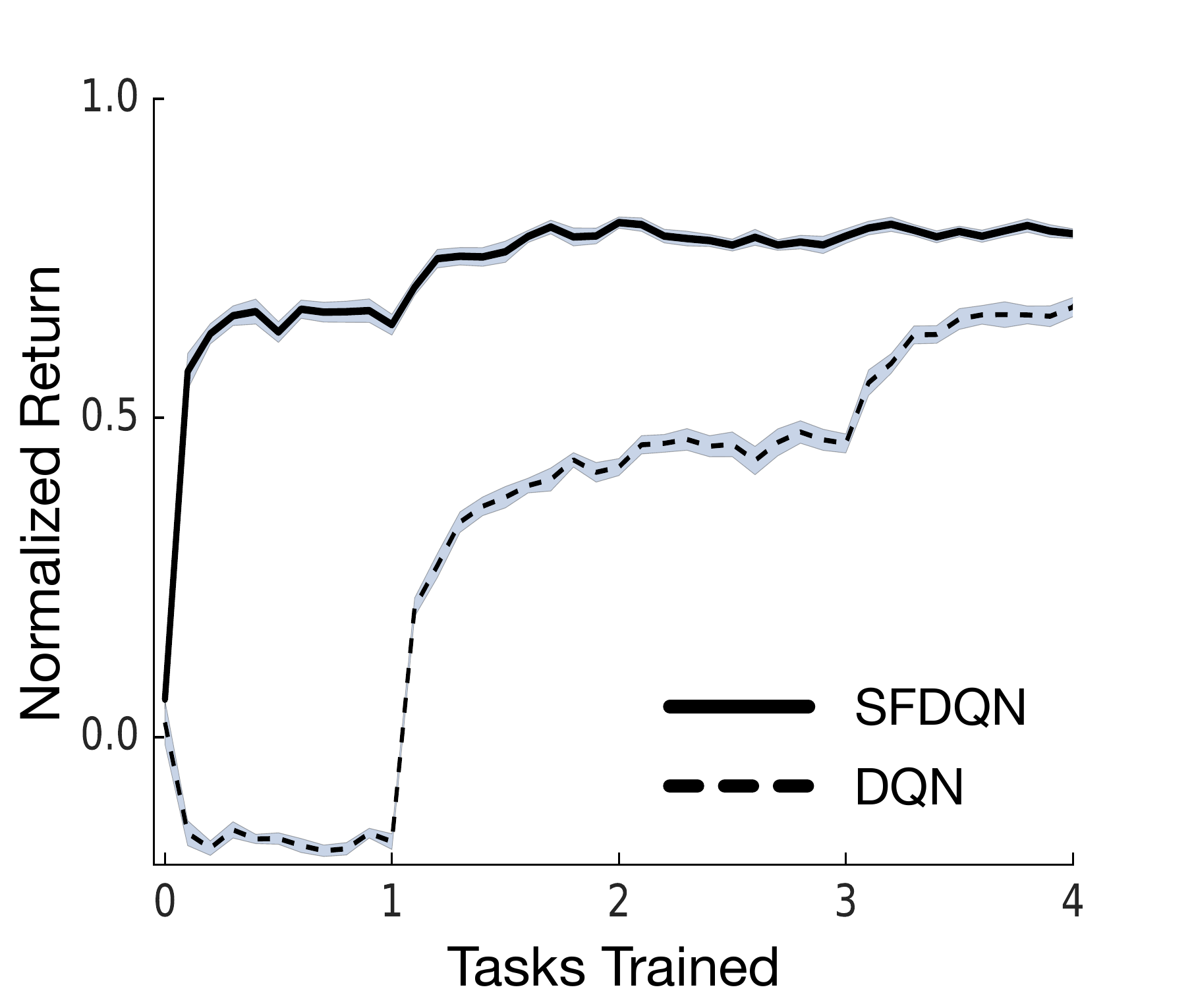} \\
 {\small (b) Average performance on test tasks. }\\ \\
  \includegraphics[scale=0.15]{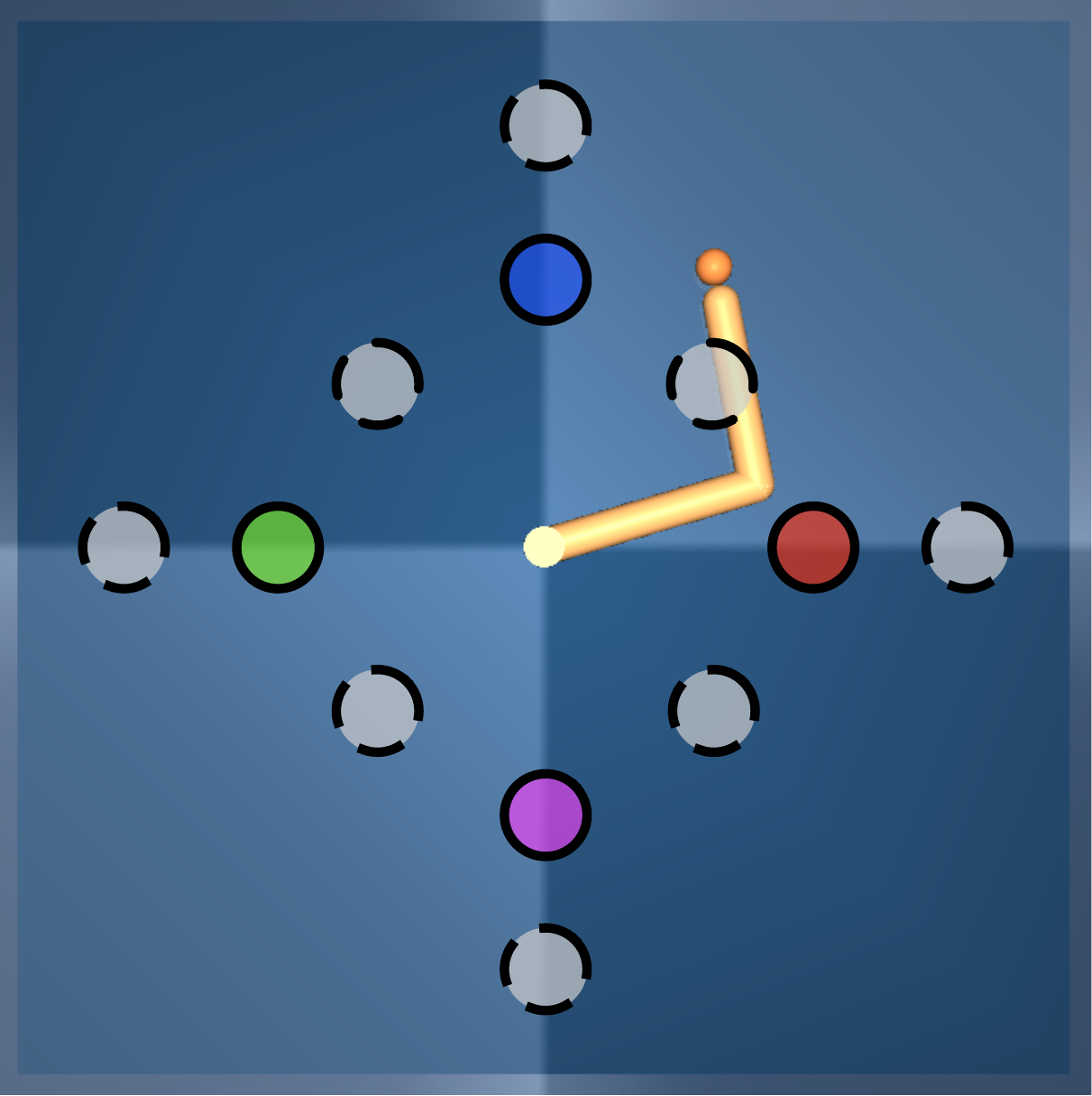} \\ 
  {\small \hspace{1.5mm} (c) \hspace{0.3mm} Colored and gray circles depict } \\
  { \small training and test targets, respectively. } 
\end{tabular}
}
\end{subfloat}
\caption{Normalized return on the reacher domain: `$1$' corresponds to the average result achieved by DQN after learning each task separately and `$0$' corresponds to the average performance of a
randomly-initialized agent (see Appendix~\ref*{sec:details_experiments} for details).
SFDQN's results were obtained using the GPI 
policies $\pi_i(s)$ defined in the text.
Shading shows one standard error over $30$ runs. 
\label{fig:results_reacher}}
\end{figure}

\section{Related work}
\label{sec:related_work}

\ctp{mehta2008transfer} approach for transfer learning is probably
the closest work to ours in the literature. 
There are important differences, though. First, \ct{mehta2008transfer}
assume that both \vphi\ and \w\ 
are always observable quantities provided by the environment.
They also focus on average reward RL, in which 
the quality of a decision policy can be characterized by a single scalar. 
This reduces the process of selecting a policy for a task 
to one decision made at the outset, which is in clear contrast with 
GPI.

The literature on transfer learning has other methods that relate to 
ours~\cp{taylor2009transfer,lazaric2012transfer}.
Among the algorithms designed for the scenario considered here, two approaches 
are particularly relevant because they also 
reuse old policies. One is \ctp{fernandez2010probabilistic} 
probabilistic policy reuse, adopted in our experiments 
and described in Appendix~\ref*{sec:details_experiments}. 
The other approach, by~\ct{bernstein99reusing}, corresponds to using our 
method but relearning all $\tvpsi^{\pi_i}$ from scratch at each new task.

When we look at SFs strictly as a representation scheme, there are clear 
similarities with \ctp{littman2001predictive} predictive state representation 
(PSR). Unlike SFs, though, PSR tries to summarize the dynamics of the entire 
environment rather than of a single policy $\pi$. A scheme that is perhaps
closer to SFs is the value function representation 
sometimes adopted in inverse RL~\cp{ng2000algorithms}. 

SFs are also related to \ctp{sutton2011horde} \emph{general 
value functions} (GVFs), which extend the notion of value function to also 
include ``pseudo-rewards.'' If we see $\phi_i$ as a pseudo-reward, 
$\psi^{\pi}_i(s,a)$ becomes a particular case of GVF.
Beyond the technical similarities, the connection between SFs and GVFs 
uncovers some principles underlying both lines of work that, when contrasted, 
may benefit both. On one hand, \ctp{sutton2011horde} and 
\ctp{modayil2014multi} hypothesis that relevant 
knowledge about the world can be expressed in the form of many 
predictions naturally translates to SFs:
if $\vphi$ is expressive enough, the agent should be able to represent 
\emph{any} relevant reward function. 
Conversely, SFs not only provide a concrete 
way of using this knowledge, they also suggest a possible criterion to
select the pseudo-rewards $\phi_i$, since
ultimately we are only interested in features that help in the 
approximation $\vphi(s,a,s')^{\top}\wt \approx r(s,a,s')$.
    
Another generalization of value functions that is related to SFs 
is \ctp{schaul2015universal} \emph{universal value function 
approximators} (UVFAs).
UVFAs extend the notion of value function to also include as an 
argument an abstract representation of a ``goal,''
which makes them particularly suitable for transfer.
The function 
$\max_j \tilde{\vpsi}^{\piexpj}(s,a)^\t \wt$ 
used in our framework can be seen as 
a function of $s$, $a$, and \wt---the latter a generic way 
of representing a goal---, and thus in some sense 
this representation \emph{is} a UVFA. 
The connection between SFs and UVFAs raises an interesting point: since under
this interpretation \wt\ is simply the description of a task,
it can in principle be a direct 
function of the observations, which opens up the possibility
of the agent determining \wt\ even \emph{before} seeing 
any rewards.

As discussed, our approach is also related to temporal abstraction and 
hierarchical RL: if we look at $\vpsi^\pi$ as instances of \ctp{sutton99between} 
\emph{options}, acting greedily with respect to the maximum
over their value functions corresponds in some sense to planning at a 
higher level of temporal abstraction 
(that is, each $\vpsi^\pi(s,a)$ is associated with an option 
that terminates after a single step). This is the view adopted by 
\ct{yao2014universal}, whose \emph{universal option model} closely resembles 
our approach in some aspects (the main difference being that they do not do 
GPI).

Finally, there have been previous attempts to combine SR and neural 
networks. \ct{kulkarni2016deep} and \ct{zhang2016deep}
propose similar architectures to jointly learn 
$\tvpsi^\pi(s,a)$, $\tvphi(s,a,s')$ and \wt.
Although neither work exploits SFs for GPI, they both discuss
other uses of SFs for transfer. In principle the proposed (or similar) 
architectures can also be used within our framework.

\section{Conclusion}
\label{sec:conclusion}

This paper builds on two concepts, both of which are generalizations of 
previous 
ideas.
The first one is SFs, a generalization of \ctp{dayan93improving}
SR that extends the original definition from discrete 
to continuous spaces and also facilitates 
the use of function approximation.  
The second concept is GPI, formalized in 
Theorem~\ref{teo:gpi}. As the name suggests, this result extends 
\ctp{bellman57dynamic} classic policy improvement theorem
from a single to multiple policies.

Although SFs and GPI 
are of interest on their own, in this paper we focus on their combination 
to induce transfer. The resulting framework is an elegant
extension of DP's basic setting that provides a 
solid foundation for transfer in RL.
As a complement to the proposed transfer approach, 
we derived a theoretical result, 
Theorem~\ref{teo:gpi_sf}, that 
formalizes the intuition that 
an agent should perform well on a novel task if it has seen a similar task
before. We also illustrated with a 
comprehensive set of experiments how the 
combination of SFs and 
GPI promotes transfer in practice.

We believe the proposed ideas lay out a general framework for
transfer in RL. By specializing the basic components
presented one can build on our 
results to derive agents able to perform well across 
a wide variety of tasks, and thus extend the range of environments that
can be successfully tackled.

\newpage

\section*{Acknowledgments}

The authors would like to thank Joseph Modayil 
for the invaluable discussions during the development of the ideas
described in this paper.
We also thank Peter Dayan, Matt Botvinick, Marc Bellemare, and Guy Lever
for the excellent comments, and Dan Horgan and Alexander 
Pritzel for their help with the experiments. Finally, we thank
the anonymous reviewers for their comments and suggestions to 
improve the paper.

\newpage

\appendix

\vspace{7mm} 
\begin{center}
\vspace{7mm} 
\noindent\makebox[\textwidth]{\rule{\textwidth}{2.0pt}} \\
\vspace{3mm} 
{\bf {\LARGE  Successor Features for \vspace{2mm} \\ Transfer in Reinforcement 
Learning} \\
\vspace{2mm} {\Large Supplementary Material} }
\vspace{5mm} 
\noindent\makebox[\textwidth]{\rule{\textwidth}{1.0pt}}

  {\bf Andr\'e Barreto}, \hspace{0.5mm}
  {\bf Will Dabney}, \hspace{0.5mm}
  {\bf R\'{e}mi Munos}, \hspace{0.5mm}
  {\bf Jonathan J. Hunt}, \hspace{0.5mm} \\
  {\bf Tom Schaul}, \hspace{0.5mm}
  {\bf Hado van Hasselt}, \hspace{0.5mm}
  {\bf David Silver} \vspace{1mm} \\ 
   \texttt{\small \{andrebarreto,wdabney,munos,jjhunt,schaul,hado,davidsilver\}@google.com} \vspace{1.5mm} \\
  DeepMind 

\end{center}

\maketitle

\begin{abstract}
In this supplement we give details of the theory and experiments that had to 
be left out of the main paper due to the space limit. For the convenience of 
the reader the statements of the theoretical results are reproduced before 
the respective proofs. We also provide a thorough description of the protocol
used to carry out our experiments, present details of the algorithms, including
their pseudo-code, and report additional empirical analysis 
that could not be included in the paper.
The numbering of sections, equations, and figures resume that used in the main 
paper, so we refer to these elements as if paper and supplement were a single 
document. We also cite references listed in the main paper.
\end{abstract}

\section{Proofs of theoretical results}
\label{sec:proofs}

\setcounter{theorem}{0}

\begin{theorem}
{\bf (Generalized Policy Improvement)} Let $\pi_1$, $\pi_2$, ..., $\pi_n$ be 
$n$ 
decision policies and 
let $\Qt^{\pi_1}$, $\Qt^{\pi_2}$, ..., $\Qt^{\pi_n}$ be approximations of their 
respective action-value functions
such that 
\begin{equation*}
\label{eq:epsilon2}
|Q^{\pi_i}(s,a) - \Qt^{\pi_i}(s,a)| \le \epsilon \, \text{ for all } s \in S, a 
\in A, \text{ and }  i \in \{1, 2, ..., n\}.
\end{equation*}
 Define
\begin{equation*}
\label{eq:pitmax2}
\pitmax(s) \in \mathop{\argmax}_a \max_i \Qt^{\pi_i}(s,a).
\end{equation*}
Then, 
\begin{equation*}
\label{eq:Qpitmax2}
\Qpitmax(s,a)  \ge \max_i Q^{\pi_i}(s,a) - \dfrac{2}{1 - \gamma} \epsilon
\end{equation*}
for any $s \in S$ and any $a \in A$,
where \Qpitmax\ is the action-value function of \pitmax.
\end{theorem}
\begin{proof}
To simplify the notation, let
\begin{equation*}
\label{eq:qtmax}
\begin{array}{ccc}
\Qmax(s,a) = \max_i Q^{\pi_i}(s,a)
& \text{ and }
& \Qtmax(s,a) = \max_i \Qt^{\pi_i}(s,a).
\end{array}
\end{equation*}
We start by noting that for any $s \in S$ and any $a \in A$ the following holds:
\begin{align*}
|\Qmax(s,a) - \Qtmax(s,a)| 
 =  |\max_i Q^{\pi_i} (s,a) - \max_i \Qt^{\pi_i}(s,a)| 
 \le  \max_i |Q^{\pi_i} (s,a) - \Qt^{\pi_i}(s,a)| 
 \le \epsilon.
\end{align*}
For all $s \in S$, $a \in A$, and $i \in \{1,2,...,n\}$ we have
\begin{align*}
\Tpitmax \Qtmax(s,a) 
& = r(s,a) + \gamma \sum_{s'} p(s' | s,a) \Qtmax(s', \pitmax(s')) \\ 
& = r(s,a) + \gamma \sum_{s'} p(s' | s,a) \max_b \Qtmax(s',b) \\
& \ge r(s,a) + \gamma \sum_{s'} p(s' | s,a)  \max_b \Qmax(s',b) - \gamma 
\epsilon \\
& \ge r(s,a) + \gamma \sum_{s'} p(s' | s,a)  \Qmax(s',\pi_i(s')) - \gamma 
\epsilon \\
& \ge r(s,a) + \gamma \sum_{s'} p(s' | s,a)  Q^{\pi_i}(s',\pi_i(s')) - \gamma 
\epsilon\\
& = T^{\pi_i} Q^{\pi_i}(s,a) - \gamma \epsilon \\
& = Q^{\pi_i}(s,a) -\gamma \epsilon. \\
\end{align*}
Since $\Tpitmax \Qtmax(s,a) \ge Q^{\pi_i}(s,a) - \gamma \epsilon$ for any $i$, 
it must be the case that 
\begin{align*}
\Tpitmax \Qtmax(s,a) 
& \ge \max_i Q^{\pi_i}(s,a) - \gamma \epsilon \\
& = \Qmax(s,a) - \gamma \epsilon \\
& \ge \Qtmax(s,a) - \epsilon - \gamma \epsilon.
\end{align*}
Let $e(s,a) = 1$ for all $s,a \in S \times A$.
It is well known that $\Tpitmax (\Qtmax(s,a) +c e(s,a)) = \Tpitmax \Qtmax(s,a) 
+ 
\gamma c$ for any $c \in \R$. 
Using this fact together with the monotonicity and contraction properties of 
the 
Bellman operator \Tpitmax,
we have 
\begin{align*}
\Qpitmax(s,a)
&  = \lim_{k \rightarrow \infty} (\Tpitmax)^k \Qtmax(s,a) \\ 
&  \ge \Qtmax(s,a) - \dfrac{1 + \gamma}{1 - \gamma} \epsilon \\
&  \ge \Qmax(s,a) - \epsilon - \dfrac{1 + \gamma}{1 - \gamma} \epsilon.
\end{align*}
\end{proof}

\begin{lemma}
\label{teo:bound_pair}
Let $\rdif_{ij} = \max_{s,a} \left|r_i(s,a) - r_j(s,a)\right|$. Then,
\begin{equation*}
\label{eq:bound}
Q_{i}^{\pi^*_i}(s,a) - Q^{\pi_j^*}_i(s,a) \le \dfrac{2  \rdif_{ij}}{1- \gamma}.
\end{equation*}
\end{lemma}

\begin{proof}
To simplify the notation, let $Q^j_i (s,a) \equiv Q_{i}^{\pi^*_j}(s,a)$.
Then,
\begin{align}
\label{eq:triangle}
\nonumber Q_i^i(s,a) - Q^j_i(s,a) 
& = Q_i^i(s,a) - Q^j_j(s,a) + Q_j^j(s,a) - Q^j_i(s,a) \\
& \le |Q_i^i(s,a) - Q^j_j(s,a)| + |Q_j^j(s,a) - Q^j_i(s,a)|. 
\end{align}
Our strategy will be to bound $|Q_i^i(s,a) - Q^j_j(s,a)|$ and 
$|Q_j^j(s,a) - Q^j_i(s,a)|$. 
Note that 
$|Q_i^i(s,a) - Q^j_j(s,a)|$ is the difference between the value functions of 
two 
MDPs with the same transition function
but potentially different rewards. 
Let $\Delta_{ij} = \max_{s,a} |Q_i^i(s,a) - Q^j_j(s,a)|$.
Then, 
\footnote{We follow the steps of \ct{strehl2005atheoretical}.}
{\small
\begin{align}
\label{eq:der_sub_bound1}
|Q_i^i(s,a) - Q^j_j(s,a)|
\nonumber & = \left|r_i(s,a) + \gamma \sum_{s'} p(s'|s,a) \max_b Q^i_i(s',b) 
- r_j(s,a) - \gamma \sum_{s'} p(s'|s,a) \max_b Q^j_j(s',b) \right| \\
\nonumber & = \left|r_i(s,a) - r_j(s,a) + \gamma \sum_{s'} p(s'|s,a) 
\left(\max_b Q^i_i(s',b) 
- \max_b Q^j_j(s',b) \right) \right| \\
\nonumber & \le \left|r_i(s,a) - r_j(s,a)\right| + \gamma \sum_{s'} p(s'|s,a) 
\left| 
\max_b Q^i_i(s',b) - \max_b Q^j_j(s',b) \right| \\
\nonumber & \le \left|r_i(s,a) - r_j(s,a)\right| + \gamma \sum_{s'} p(s'|s,a) 
\max_b \left| Q^i_i(s',b) - Q^j_j(s',b) \right| \\
& \le \rdif_{ij} + \gamma \Delta_{ij}. 
\end{align}
}
Since~(\ref{eq:der_sub_bound1}) is valid for any $s,a \in S \times A$, we have 
shown that 
$\Delta_{ij} \le \rdif_{ij} + \gamma \Delta_{ij}$.
Solving for $\Delta_{ij}$ we get 
\begin{equation}
\label{eq:sub_bound1}
\Delta_{ij} \le \dfrac{1}{1- \gamma} \rdif_{ij}.
\end{equation}
We now turn our attention to $|Q_j^j(s,a) - Q^j_i(s,a)|$. Following the 
previous 
steps, 
define $\Delta'_{ij} = \max_{s,a} |Q_i^i(s,a) - Q^j_i(s,a)|$.
Then,
\begin{align*}
\label{eq:der_sub_bound2}
|Q_j^j(s,a) - Q^j_i(s,a)|
\nonumber & = \left|r_j(s,a) + \gamma \sum_{s'} p(s'|s,a) Q^j_j(s',\pi_j^*(s')) 
- r_i(s,a) - \gamma \sum_{s'} p(s'|s,a) Q^j_i(s',\pi_j^*(s')) \right| \\
\nonumber & = \left|r_i(s,a) - r_j(s,a) + \gamma \sum_{s'} p(s'|s,a) 
\left(Q^j_j(s',\pi_j^*(s')) - Q^j_i(s', \pi_j^* (s')) \right) \right| \\
\nonumber & \le \left|r_i(s,a) - r_j(s,a)\right| + \gamma \sum_{s'} p(s'|s,a) 
\left| 
 Q^j_j(s',\pi_j^*(s')) - Q^j_i(s',\pi_j^*(s')) \right| \\
& \le \rdif_{ij} + \gamma \Delta'_{ij}. 
\end{align*}
Solving for $\Delta'_{ij}$, as above, we get
\begin{equation}
\label{eq:sub_bound2}
\Delta'_{ij} \le \dfrac{1}{1- \gamma} \rdif_{ij}.
\end{equation}
Plugging~(\ref{eq:sub_bound1}) and~(\ref{eq:sub_bound2})
back in~(\ref{eq:triangle}) we get the desired result.
\end{proof}

\begin{theorem}
Let $M_i \in \M$ and let $Q^{\pi_j^{*}}_i$ be the value function of an optimal 
policy of $M_j \in \M$ when executed in $M_i$. Given the set 
$\{\Qt^{\pi_1^{*}}_i, \Qt^{\pi_2^{*}}_i, ..., \Qt^{\pi_n^{*}}_i\}$ such that 
\begin{equation*}
 \left|Q^{\pi_j^*}_i(s, a) - \Qt^{\pi_j^{*}}_i(s,a) \right| \le \epsilon  
\, \text{ for all } s \in S, a \in A, \text{ and }  j \in \{1, 2, ..., n\},
\end{equation*}
let 
\begin{equation*}
\pihmax(s) \in \mathop{\argmax}_a \max_j \Qt^{\pi_j^{*}}_i (s,a).
\end{equation*}
Finally, let $\phimax = \max_{s,a} ||\vphi(s,a)||$, 
where $||\cdot||$ is 
the norm induced by the inner product adopted.
Then,
\begin{equation*}
Q^*_i(s,a) - Q^{\pi}_i(s,a) 
\le \dfrac{2 }{1- \gamma}  \left(\phimax \, \min_j || \w_i - \w_j|| + \epsilon 
\right).
\end{equation*}
\end{theorem}
\begin{proof}
The result is a direct application of Theorem~\ref{teo:gpi} and 
Lemma~\ref{teo:bound_pair}.
For any $j \in \{1, 2, ..., n\}$, we have 
\begin{equation*}
\begin{array}{llr}
Q_i^{*}(s,a) - Q^{\pi}_i(s,a)  
& \le Q_i^{*}(s,a) - Q^{\pi_j^*}_i(s,a)  + \dfrac{2}{1 - \gamma} \epsilon& 
\text{ (Theorem~\ref{teo:gpi}) } \\
& \le \dfrac{2}{1 - \gamma} \max_{s,a} |r_i(s,a) - r_j(s,a) | 
 + \dfrac{2}{1 - \gamma} \epsilon 
& \text{ (Lemma~\ref{teo:bound_pair}) } \\
& = \dfrac{2}{1 - \gamma} \max_{s,a} |\vphi(s,a)^\t \w_i - \vphi(s,a)^\t \w_j | 
 
+ \dfrac{2}{1 - \gamma} \epsilon \\
& = \dfrac{2}{1 - \gamma} \max_{s,a} |\vphi(s,a)^\t (\w_i -  \w_j) |  
+ \dfrac{2}{1 - \gamma} \epsilon \\
& \le \dfrac{2}{1 - \gamma} \max_{s,a} ||\vphi(s,a)|| \, ||\w_i -  \w_j||  
+ \dfrac{2}{1 - \gamma} \epsilon 
& \text{ (Cauchy-Schwarz's inequality) } \\
& = \dfrac{2 \phimax}{1 - \gamma} ||\w_i -  \w_j||
+ \dfrac{2}{1 - \gamma} \epsilon.  \\
\end{array}
\end{equation*}
\end{proof}

\section{Details of the experiments}
\label{sec:details_experiments}

In this section we provide additional information about our experiments. 
We start with the four-room environment and then we discuss the reacher domain.
In both cases the structure of the discussion is the same:
we start by giving a more in depth description of the environment itself, both at 
a conceptual level and at a practical level, then
we provide a thorough description of the algorithms used, and, finally, 
we explain the protocol used to carry out the experiments.

\subsection{Four-room environment}

\subsubsection{Environment}
\label{sec:environment}

In Section~\ref{sec:experiments} of the paper we gave an intuitive description 
of the four-room domain used in our experiments. In this section we provide a more 
formal definition of the environment \MM\ as a family of Markov decision 
processes (MDPs) $M$, each one associated with a task.

The environment has objects that can 
be picked up by the agent by passing over them. 
There is a total of $n_o$ objects, each belonging to 
one of $n_c \le n_o$ classes.
The class of an object determines the reward 
$r_c$ associated with it. 
An episode ends when the agent reaches the goal, upon which all 
the objects re-appear. 
We assume that $r_g$ is always $1$ but $r_c$ may 
vary: a specific instantiation of the rewards $r_c$ 
defines a \emph{task}. 
Every time a new task starts the rewards $r_c$ are sampled 
from a uniform distribution over $[-1,1]$.
Figure~\ref{fig:forage_maze}
shows the specific environment layout used, 
in which $n_o = 12$ and $n_c = 3$.

We now focus on a single task $M \in \MM$.
We start by describing the state and action spaces, and the transition dynamics.
The agent's position at any instant in time is a point $\{s_x, s_y\} \in 
[0,1]^2$. There are four actions available, 
$\A \equiv \{\mathrm{up}, \mathrm{down},\mathrm{left}, \mathrm{right}\}$.
The execution of one of the actions moves the agent $0.05$ units in the desired
direction, and normal random noise with zero mean and standard deviation 
$0.005$ is added to the position of the agent (that is, a move 
along the $x$ axis would be $s'_x = s_x \pm \Normal(0.05, 0.005)$, where 
$\Normal(0.05, 0.005)$ is a normal variable with mean $0.05$ and standard 
deviation $0.005$). If after a move the agent ends up outside of the four rooms 
or on top of a wall the move is undone. 
Otherwise, if the agent lands on top of an object it picks it up, and if it 
lands on the goal region the episode is terminated (and all objects 
re-appear).
In the specific instance of the environment shown in 
Figure~\ref{fig:forage_maze} objects
were implemented as circles of radius $0.04$, the goal is a circle of 
radius $0.1$ centered at one of the extreme points of the map, and the walls
are rectangles of width $0.04$ traversing the environment's range.

As described in the paper, within an episode each 
of the $n_o$ objects can be present or 
absent, and since they define the reward function 
a well defined Markov state must distinguish between all 
possible $2^{n_o}$ object configurations. Therefore, the state space of our MDP 
is $\S \equiv \{0,1\}^{n_o} \times \R^{2}$.
An intuitive way of visualizing \S\ is to note that 
each of the $2^{n_o}$ object configurations is potentially associated with a 
different value function over the continuous space $[0,1]^2$.

Having already described \S, \A, and $p(\cdot|s,a)$,
we only need to define the reward function $R(s,a,s')$ and 
the discount factor $\gamma$ in order to conclude the formulation of the MDP 
$M$. As discussed in Section~\ref{sec:experiments},
the reward $R(s,a,s')$ is a deterministic function of $s'$: if the agent 
is over an object of class $c$ in $s'$ it gets a reward of $r_c$, and if it is 
in the goal region it gets a reward of $r_g = 1$; in all other cases the reward 
is zero. In our experiments we fixed $\gamma = 0.95$.

By mapping each object onto its class, it is possible to 
construct features $\vphi(s, a, s')$ that perfectly predicts the reward for 
all tasks $M$ in the environment \M, as in~(\ref{eq:reward})
and~(\ref{eq:M}). 
Let $\vphi_c(s, a, s') \equiv 
\II\{$is the agent over an 
object of class $c$ in $s'$?$\}$, where $\I{\mathrm{false}} = 0$ and 
$\I{\mathrm{true}} = 1$. Similarly, let $\vphi_g(s, a, s') \equiv \II\{$is 
the agent over the goal region in $s'$?$\}$.
By concatenating the $n_c$ functions $\vphi_c$ and $\vphi_g$ we get the 
vector $\vphi(s, a, s') \in \{0,1\}^{n_c + 1}$; 
now, if we make $\w_c = r_c$ for all $c$ and $\w_{n_c + 1} = r_g$, 
it should be clear that $r(s,a,s') = \vphi(s, a, s')^\t \w$, as desired. 

Since $r(s,a,s')$ can be written in the form of ~(\ref{eq:reward}),
the definition of $M$ can be naturally extended to \MM, as in~(\ref{eq:M}).
In our experiments we assume that the agents receive a signal from \MM\ 
whenever the task changes (see Algorithms~\ref{alg:ql},~\ref{alg:ppr}, 
and~\ref{alg:sfql} and discussion below).

\subsubsection{Algorithms}
\label{sec:algorithms}

We assume that the agents know their position $\{s_x, s_y\} \in [0,1]^2$ and 
also have an ``object detector'' $\od \in \{0,1\}^{n_o}$ whose \ith\ component 
is $1$ if and only if the agent is over object $i$.
Using this information the agents build two vectors of features.
The vector $\vfp(s) \in \R^{100}$ is composed of the activations of a regular  
$10 \times 10$ grid of radial basis functions at the point $\{s_x, s_y\}$. 
Specifically, in our experiments we used Gaussian functions, that is:
\begin{equation}
\label{eq:gaussian}
\vf_{pi}(s) = \exp\left(-\dfrac{(s_x -\mat{c}_{i1})^2 +  (s_y 
-\mat{c}_{i2})^2}{\sigma}\right),
\end{equation}
where $\mat{c}_i \in \R^{2}$ is the center of the \ith\ Gaussian. 
As explained in Section~\ref{sec:exp_setup}, the value of 
$\sigma$ was determined in preliminary experiments with QL; all algorithms 
used $\sigma = 0.1$. 
In addition to $\vfp(s)$, using \od\ the agents build an ``inventory'' $\vfi(s) 
\in 
\{0,1\}^{n_o}$
whose \ith\ component indicates whether the \ith\ object has been 
picked up or not. 
The concatenation of $\vfi(s)$ and $\vfp(s)$ 
plus a constant term gives rise to the feature vector $\vf(s) \in \R^D$ 
used by all the agents to represent the value function:
$\tilde{Q}^{\pi}(s,a) = \vf(s)^\t\vweights_a^{\pi}$, where 
$\vweights_a^{\pi} \in \R^D$ are learned weights. 

It is instructive to take a closer look at how exactly SFQL represents 
the value function. Note that, even though our algorithm also represents 
$\tilde{Q}^{\pi}$ as a linear combination of the features $\vf(s)$, it never 
explicitly computes $\vweights^{\pi}_a$. Specifically, SFQL represent SFs as 
$\tilde{\vpsi}^{\pi}(s,a) = \vf(s)^\t \Z^\pi_{a}$, 
where $\Z^\pi_{a} \in \R^{D \times d}$, and the value function 
as $\tilde{Q}^{\pi}(s,a) = \tilde{\vpsi}^{\pi}(s,a)^\t \wt = 
\vf(s)^\t \Z^\pi_{a} \wt$. By making $\vweights^\pi_a = \Z^\pi_{a} \wt$,
it becomes clear that SFQL unfolds the problem of learning $\vweights^\pi_a$ 
into the sub-problems of learning $\Z^\pi_{a}$ and \wt. These 
parameters are learned via gradient descent 
in order to minimize losses induced 
by~(\ref{eq:bellman_psi}) and~(\ref{eq:reward}), respectively (see below).

The pseudo-codes of QL, PRQL, and SFQL are given in 
Algorithms~\ref{alg:ql}, \ref{alg:ppr}, and \ref{alg:sfql}, respectively. 
As one can see, all algorithms used an $\epsilon$-greedy policy 
to explore the environment, 
with $\epsilon = 0.15$~\cp{sutton98reinforcement}. 
Two design choices deserve to be discussed here. First, as mentioned in 
Section~\ref{sec:environment}, the agents ``know'' when the task changes. This 
makes it possible for the algorithms to take measures like reinitializing the 
weights $\vprm^{\pi}_a$ or adding a new representative to the set of decision 
policies. 
Another design choice, this one specific to PRQL and SFQL, was not to limit 
the number of decision policies (or $\tvpsi^{\pi_i}$) stored.
It is not difficult to come up with strategies to avoid both 
an explicit end-of-task signal and an ever-growing set of policies. For 
example, in Section~\ref{sec:gpi_sf} we discuss how \wt\ can be 
used to select which $\tvpsi^{\pi_i}$ to keep in the case of limited memory. 
Also, by monitoring the error in the approximation 
$\tvphi(s,a,s')^\t\wt$ one can detect when the task has 
changed in a significant way. Although these are interesting extensions, 
given the introductory character of this paper 
we refrained from overspecializing the algorithms 
in order to illustrate the properties of the proposed approach 
in the clearest way possible.

\begin{algorithm}[t]
   \caption{QL} 
   \label{alg:ql}
\begin{algorithmic}[1]
\REQUIRE 
\begin{tabular}{cl}
$\epsilon$ & exploration parameter for $\epsilon$-greedy strategy\\
$\alpha$ & learning rate\\
\end{tabular}
\FOR{$t \la 1, 2, ..., \mathrm{num\_tasks}$}
\STATE $\vprm_a \la $ small random initial values {\bf for} all $a \in \A$
\label{it:initialize_z}
\STATE $\mathrm{new\_episode} \la \mathrm{true}$
\FOR{$i \la 1, 2, ..., \mathrm{num\_steps}$}
\IF{$\mathrm{new\_episode}$}
\STATE $\mathrm{new\_episode} \la \mathrm{false}$
\STATE $s \la $ initial state
\ENDIF
\STATE $\mathrm{sel\_rand\_a} \sim \mathrm{Bernoulli}(\epsilon)$
\COMMENT{Sample from a Bernoulli distribution with parameter $\epsilon$}
\STATE {\bf if} {$\mathrm{sel\_rand\_a}$} {\bf then}  
$a \sim \mathrm{Uniform}\left(\{1, 2, ..., |A|\}\right)$ 
\COMMENT{$\epsilon$-greedy exploration strategy}
\STATE {\bf else} $a \la \argmax_{b} 
Q(s,b)$ 
\STATE Take action $a$ and observe reward $r$ and next state $s'$
\IF{$s'$ is a terminal state}
\STATE $\gamma \la 0$
\STATE $\mathrm{new\_episode} \la \mathrm{true}$
\ENDIF
\STATE $\vprm_a \la \vprm_a + \alpha {\left(r + \gamma \max_{a'} 
Q(s', a') - Q(s,a)\right)}\nabla_{\vprm} Q(s,a)$ \label{it:ql_update}
\STATEx \COMMENT{For $Q(s,a) = \vf(s)^\t \vprm_a$, $\nabla_{\vprm} Q(s,a) = \vf(s)$}
\STATE $s \la s'$
\ENDFOR
\ENDFOR
\end{algorithmic}
\end{algorithm}

\begin{algorithm}[t]
   \caption{PRQL} 
   \label{alg:ppr}
\begin{algorithmic}[1]
\REQUIRE 
\begin{tabular}{cl}
$\epsilon$ & exploration parameter for $\epsilon$-greedy strategy\\
$\alpha$ & learning rate \\
$\eta$ & parameter to control probability to reuse old policy\\ 
$\tau$ & parameter to control bias in favor of stronger policies\\ 
\end{tabular}
\FOR{$t \la 1, 2, ..., \mathrm{num\_tasks}$}
\FOR {$k \la 1, 2, ..., t$} 
\STATE $\score_k \la 0$
\COMMENT{$\score_k$ is the score associated with policy $\pi_k$}
\STATE $\ntpu_k \la 0$
\COMMENT{$\ntpu_k$ is the number of times policy $\pi_k$ was 
used}
\ENDFOR
\STATE $\act \la t$
\COMMENT{$\act$ is the index of the policy currently being used}
\STATE $\vprm^t_a \la $ small random initial values 
\STATE $\currscore \la 0$
\STATE $\mathrm{new\_episode} \la \mathrm{true}$
\FOR{$i \la 1, 2, ..., \mathrm{num\_steps}$}
\IF{$\mathrm{new\_episode}$}
\STATE $\score_{\act} \la \dfrac{\score_{\act} \times \ntpu_{\act} + 
\currscore}{\ntpu_{\act} + 1}$
\COMMENT{Update score for policy currently being used}
\STATE {\bf for} {$k \la 1, 2, ..., t$} {\bf do} $p_{k} \la e^{\tau \times 
\score_{k}} / {\sum_{j=1}^{t} e^{\tau \times \score_j}} $
\COMMENT{Turn scores into probabilities}
\STATE $c \sim \mathrm{Multinomial}(p_1, p_2, ..., p_t)$
\COMMENT{Select policy $c$ with probability $p_c$}
\STATE $\ntpu_{\act} \la \ntpu_{\act} + 1$
\COMMENT{Update number of times policy $\pi_c$ has been used}
\STATE $\currscore \la 0$
\STATE $\mathrm{new\_episode} \la \mathrm{false}$
\STATE $s \la $ initial state
\ENDIF
\STATE {\bf if} {$t \ne c$} {\bf then} $\mathrm{use\_prev\_policy} \sim 
\mathrm{Bernoulli}(\eta)$
{\bf else} $\mathrm{use\_prev\_policy} \la \mathrm{false}$
\IF[Action will be selected by $\pi_c$, the policy being 
reused]{$\mathrm{use\_prev\_policy}$}
\STATE $a \la \argmax_{a'} Q_{\act}(s,a')$
\ELSE[Action will be selected by $\pi_c$, the most recent policy]
\STATE $\mathrm{sel\_rand\_a} \sim \mathrm{Bernoulli}(\epsilon)$
\STATE {\bf if} $\mathrm{sel\_rand\_a}$ {\bf then} $a \sim 
\mathrm{Uniform}\left(\{1, 2, ..., 
|A|\}\right)$ {\bf else} $a \la \argmax_{a'} Q_t(s,a')$ 
\STATEx \COMMENT{$\epsilon$-greedy exploration strategy}
\ENDIF
\STATE Take action $a$ and observe reward $r$ and next state $s'$
\IF{$s'$ is a terminal state}
\STATE $\gamma \la 0$
\STATE $\mathrm{new\_episode} \la \mathrm{true}$
\ENDIF
\STATE $\vprm^t_a \la \vprm^t_a + \alpha {\left(r + \gamma \max_{a'} 
Q_t(s', a') - Q_t(s,a)\right)}\nabla_{\vprm} Q_t(s,a)$ 
\STATEx \COMMENT{For $Q_t(s,a) = \vf(s)^\t \vprm^t_a$, $\nabla_{\vprm} Q_t(s,a) = 
\vf(s)$}
\STATE $\currscore \la \currscore + r$
\STATE $s \la s'$
\ENDFOR
\ENDFOR
\end{algorithmic}
\end{algorithm}

\begin{algorithm}[t]
   \caption{SFQL} 
   \label{alg:sfql}
\begin{algorithmic}[1]
\REQUIRE 
\begin{tabular}{cl}
$\epsilon$ & exploration parameter for $\epsilon$-greedy strategy\\
$\alpha$ & learning rate for \vpsi's parameters \\
$\alpha_w$ & learning rate for \w\ \\
$\vphi$ & features to be predicted by SFs \\
\end{tabular}
\FOR{$t \la 1, 2, ..., \mathrm{num\_tasks}$}
\STATE $\w_t \la $ small random initial values \label{it:keep_w} 
\STATE $\Z^t_a \la $ small random initial values in $\R^{D \times h}$
{\bf if} $t = 1$ {\bf else} $\Z^{t-1}_a$ 
\STATEx 
\COMMENT{The \kth\ column of $\Z^t_a$, $\z^{tk}_a$, are the 
parameters of the \kth\ component of $\tvpsi_t$, $(\tvpsi_t)_k \equiv \tvpsi_{tk}$}
\STATE $\mathrm{new\_episode} \la \mathrm{true}$
\FOR{$i \la 1, 2, ..., \mathrm{num\_steps}$}
\IF{$\mathrm{new\_episode}$}
\STATE $\mathrm{new\_episode} \la \mathrm{false}$
\STATE $s \la $ initial state
\ENDIF
\STATE  $c \la \argmax_{k \in \{1, 2, ..., t\}} 
\max_{b} \tvpsi_k(s,b)^{\t} \w_t$
\label{it:gpi_action_selection}
\STATEx \COMMENT{$c$ is the index of the \tvpsi\ associated with the largest value in 
$s$}
\STATE $\mathrm{sel\_rand\_a} \sim \mathrm{Bernoulli}(\epsilon)$
\COMMENT{Sample from a Bernoulli distribution with parameter $\epsilon$}
\STATE {\bf if \label{it:epsilon_greedy_n}} {$\mathrm{sel\_rand\_a}$} {\bf then}  
$a \sim \mathrm{Uniform}\left(\{1, 2, ..., |A|\}\right)$ 
\COMMENT{$\epsilon$-greedy exploration strategy} 
\STATE {\bf else  \label{it:epsilon_greedy}} $a \la \argmax_{b} 
\tvpsi_{c}(s,b)^{\t} \w_t$ 
\STATE Take action $a$ and observe reward $r$ and next state $s'$
\IF{$s'$ is a terminal state}
\STATE $\gamma \la 0$
\STATE $\mathrm{new\_episode} \la \mathrm{true}$
\ELSE 
\STATE  $a' \la \argmax_b\max_{k \in \{1, 2, ..., t\}}\tvpsi_k(s',b)^{\t} \w_t$
\COMMENT{$a'$ is the action with the highest value in $s'$}
\ENDIF
\STATE $\w_t \la \w_t + \alpha_w {\left[r - \vphi(s, a, s')^{\t} 
\w \right]} \vphi(s, a, s')$ 
\COMMENT{Update $\w$ \label{it:update_w}}
\FOR{$k \la 1, 2, ..., d$}
\STATE $\z^{tk}_a \la \z^{tk}_a + \alpha {\left[\vphi_k(s, a, s') + \gamma 
\tvpsi_{tk}(s', a') - \tvpsi_{tk}(s,a)\right]}\nabla_{\z} 
\tvpsi_{tk}(s,a)$ 
\label{it:update_psi}
\STATEx \COMMENT{For $\tvpsi_t(s,a) = \vf(s)^\t \Z^t_a$, $\nabla_{\vprm} 
\tvpsi_{tk}(s,a) = \vf(s)$}
\ENDFOR
\IF{$c \ne t$ \label{it:update_previous_psi_b} }
\STATE {$a' \la \argmax_{b} \tvpsi_c(s',b)^{\t} \w_c$ \label{it:gpi_action_selection2}}
\COMMENT{$a'$ is selected according to reward function induced by $\w_c$}
\FOR{$k \la 1, 2, ..., d$}
\STATE $\z^{ck}_a \la \z^{ck}_a + \alpha {\left[\vphi_k(s, a, s') + \gamma 
\tvpsi_{ck}(s', a') - \tvpsi_{ck}(s,a)\right]}\nabla_{\z} 
\tvpsi_{ck}(s,a)$ 
\COMMENT{Update $\tvpsi_c$ \label{it:update_previous_psi}}
\ENDFOR
\ENDIF \label{it:update_previous_psi_e}
\STATE $s \la s'$    
\ENDFOR
\ENDFOR
\end{algorithmic}
\end{algorithm}

\subsubsubsection{{\bf SFQL}}

We now discuss some characteristics of the specific SFQL algorithm 
used in the experiments. 
First, note that errors in the value-function approximation 
can potentially have a negative effect on GPI, since an overestimated 
$\Qt^{\pi_i}(s,a)$ may be the function determining the action selected
by $\pi$ in~(\ref{eq:pitmax}). 
One way of keeping this phenomenon from occurring indefinitely is 
to continue to update the functions $\Qt^{\pi_i}(s,a)$ that are relevant for 
action selection. In the context of SFs this corresponds to 
constantly refining $\tvpsi^{\pi_i}$, which can be done as long as we 
have access to $\pi_i$. 
In the scenario considered here we can recover $\pi_i$ by keeping 
the weights $\wt_i$ used to learn the SFs $\tvpsi^{\pi_i}$
(line~\ref{it:keep_w} of Algorithm~\ref{alg:sfql}). Note that with this information
one can easily update any $\tvpsi^{\pi_i}$ off-policy; as shown 
in lines~\ref{it:update_previous_psi_b}--\ref{it:update_previous_psi_e} of Algorithm~\ref{alg:sfql}, in the version of SFQL used in the 
experiments we 
always update the $\tvpsi^{\pi_i}$ that achieves the maximum
in~(\ref{eq:pitmax}) (line~\ref{it:gpi_action_selection} 
of the pseudo-code). 

Next we discuss the details of how \wt\ and $\tvpsi^{\pi}$ are 
learned.
We start by showing the loss function used to compute \wt: 
\begin{equation}
\label{eq:loss_w}
\mathrm{L}_w(\wt) = \E_{(s,a,s') \sim \mathcal{D}}\left[ 
\left(r(s,a,s') - \tvphi(s,a,s')^{\t} \wt\right)^2\right],
\end{equation}
where $\mathcal{D}$ is a distribution over $\S \times \A \times \S$
which in RL is usually the result of executing a policy 
under the environment's dynamics $p(\cdot|s,a)$. The minimization 
of~(\ref{eq:loss_w}) is done in line~\ref{it:update_w} of 
Algorithm~\ref{alg:sfql}. As discussed, SFQL keeps a set of $\tvpsi^{\pi_i}$,
each one associated with a policy $\pi_i$. The loss function used to compute 
each $\tvpsi^{\pi_i}$ is 
\begin{equation}
\label{eq:loss_psi}
\begin{array}{cl}
\mathrm{L}_Z(\tvpsi^{\pi_i}) \equiv  \mathrm{L}_Z(\Z^{\pi_i}_{a})  
& = \E_{(s,a, s') \sim \mathcal{D}}\left[\left(\tvphi(s,a,s')
+ \gamma \tvpsi^{\pi_i}(s',a') - \tvpsi^{\pi_i}(s,a) \right)^2\right] \\
& = \E_{(s,a, s') \sim \mathcal{D}}\left[\left(\tvphi(s,a,s')
+ \gamma \vf(s')^\t \Z^{\pi_i}_{a'} - \vf(s)^\t \Z^{\pi_i}_{a} 
\right)^2\right], \\
\end{array}
\end{equation}
where $a'  = \argmax_b \Qt^{\pi_i}(s',b) = \argmax_b 
\tvpsi^{\pi_i}(s',b)^\t\wt$.   
Note that the policy that induces $\mathcal{D}$ is not necessarily 
$\pi_i$---that is, $\tvpsi^{\pi_i}(s,a)$ can be learned off-policy, as 
discussed 
above. 
As usual in RL, the target 
$\tvphi(s,a,s') + \gamma \vf(s')^\t \Z^{\pi_i}_{a'}$ is considered fixed, 
{\sl i.e.}, the loss $\mathrm{L}_Z$ is minimized with respect to 
$\Z^{\pi_i}_{a}$ only.
The minimization of~(\ref{eq:loss_psi}) is done in 
lines \ref{it:update_psi}  
and \ref{it:update_previous_psi} of Algorithm~\ref{alg:sfql}. 

As discussed in Section~\ref{sec:experiments}, 
we used two versions of SFQL in our experiments. In the first one, 
SFQL-$\vphi$, we assume that the agent knows how to construct 
a vector of features \vphi\ that perfectly predicts the reward for 
all tasks $M$ in the environment \M. 
The other version of our 
algorithm, SFQL-$h$, uses an approximate \tvphi\ learned from data. 
We now give details of how \tvphi\ was computed in this case.
In order to learn \tvphi, we used the samples $(s_i, 
a_i, r_i, s'_i)_t$ collected by QL in the first $t = 1, 2, ..., 20$ tasks. 
Since 
this results in an unbalanced dataset in 
which most of the transitions have $r_i=0$, we kept all the samples with 
$r_i \ne 0$ 
and discarded $75\%$ of the remaining samples. We then used the resulting 
dataset to minimize the following loss:
\begin{equation}
\label{eq:loss_phi}
\begin{array}{cl}
\mathrm{L}_H(\tvphi) \equiv  \mathrm{L}_H(\H, \w_t)  
= \E_{(s, s', r) \sim 
\mathcal{D}'_t}\left[\left(
\varsigma(\vf(s,s')^\t\H)^\t\w_t -r\right) ^2\right] 
\; \text{ for } t = 1, 2, ..., 20,\\
\end{array}
\end{equation}
where $\mathcal{D}'_t$ reflects the 
down-sampling of zero rewards. 
The vector of features 
$\vf(s,s') $ is the concatenation of $\vf(s)$ and $\vf(s')$, and 
$\varsigma(\cdot)$ is a sigmoid function applied element-wise. 
We note that $\od(s') = \vf_i(s') - \vf_i(s)$, from which it is possible to 
compute an ``exact'' $\tvphi = \vphi$. In order to 
minimize~(\ref{eq:loss_phi})
we used the multi-task framework proposed by~\ct{caruana97multitask}.
Simply put, \ctp{caruana97multitask} approach consists in 
looking at $\varsigma(\vf(s,s')^\t\H)^\t\w_t$ as a neural network with one 
hidden layer and $20$ outputs, that is, 
\wt\ is replaced with $\tilde{\mat{W}} \in \R^{h \times 20}$ and 
a reward $r$ received in the \th{t} task is extended into a 
$20$-dimensional vector in which the \th{t} component is $r$ and 
all other components are zero. One can then minimize~(\ref{eq:loss_phi}) 
with respect to the parameters \H\ and $\w_t$ 
through gradient descent.

Although this strategy of using the $k$ first tasks to learn $\tvphi$ is 
feasible, in practice one may want to replace this arbitrary decision 
with a more adaptive approach, such as updating $\tvphi$ online
until a certain stop criterion is satisfied. Note though that a 
significant change in \tvphi\ renders the SFs $\tvpsi^{\pi_i}$ outdated, and 
thus the benefits of refining the former should be weighed against 
the overhead of constantly updating the latter, potentially off-policy.

As one can see in this section, we tried to keep the methods 
as simple as possible in order to 
not obfuscate the main message of the paper, which is 
not to propose any particular algorithm but rather 
to present a general framework for transfer 
based on the combination of SFs and GPI.

\subsubsection{Experimental setup}
\label{sec:exp_setup}

In this section we describe the precise protocol adopted to carry out our 
experiments. Our initial step was to use QL, the basic algorithm 
behind all three algorithms, to make some decisions that apply to all of them. 
First, in order to define the features $\vf(s)$ used by the 
algorithms, we checked the performance of QL when using different 
configurations of the vector $\vf_p(s)$ giving the position of the agent. 
Specifically, we tried two-dimensional grids of Gaussians with $5$, $10$, $15$, 
and $20$ functions per dimension. Since the improvement in QL's performance 
when using a number of functions larger than $10$ was not very significant, we 
adopted a $10 \times 10$ grid of Gaussians in our experiments. We also varied 
the value of the parameter $\sigma$ appearing in~(\ref{eq:gaussian}) in the set 
$\{0.01, 0.1, 0.3\}$. Here the best performance of QL was obtained with $\sigma 
= 0.1$,  which was then the value adopted throughout the experiments. 
The parameter $\epsilon$ used for $\epsilon$-greedy exploration was also 
set based on QL's performance.
Specifically, we varied 
$\epsilon$ in the set $\{0.15, 0.2, 0.3\}$ and verified that the best 
results were obtained with $\epsilon = 0.15$ (we tried relatively large 
values for $\epsilon$ because of the non-stationarity of the underlying
environment). Finally, we tested two variants of QL: one that resets 
the weights $\vprm^{\pi}_a$ every time a new task starts,
as in line~\ref{it:initialize_z} of Algorithm~\ref{alg:ql}, 
and one that keeps the old values. Since the performance of the former was 
significantly better, we adopted this strategy for all algorithms.

QL, PRQL, and SFQL depend on different sets of parameters, as shown in 
Algorithms~\ref{alg:ql}, \ref{alg:ppr}, and \ref{alg:sfql}.
In order to properly configure the algorithms we tried three different values 
for each parameter and checked the performance of the 
corresponding agents under each 
resulting configuration. Specifically, we tried the following sets of values 
for each parameter:
\begin{center}
\begin{tabular}{c|l|l}
{\bf Parameter} & {\bf Algorithms} & {\bf Values} \\ \hline 
$\alpha$   & QL, PRQL, SFQL & $\{0.01, 0.05, 0.1\}$ \\
$\alpha_w$ & SFQL           & $\{0.01, 0.05, 0.1\}$ \\
$\eta$     & PRQL           & $\{0.1, 0.3, 0.5\}$ \\
$\tau$     & PRQL           & $\{1, 10, 100\}$ \\
\end{tabular}
\end{center}
The cross-product of the values above resulted in $3$ configurations of QL, 
$27$ configurations of PRQL, and $9$ configurations of SFQL. 
The results reported correspond to the best performance of each algorithm, that 
is, for each algorithm we picked the configuration that lead to 
the highest average return over all tasks.

\subsection{Reacher environment}
\label{sec:reacher}

\subsubsection{Environment}

The reacher environment is a two-joint 
torque-controlled robotic arm simulated using 
the MuJoCo physics engine~\cp{todorov2012mujoco}. 
It is based on one of the domains used 
by~\ct{lillicrap2015continuous}. 
This is a particularly appropriate domain to illustrate our ideas because it is  straightforward to define multiple tasks (goal locations) sharing the same 
dynamics. 

The problem's state space $\S \subset \R^{4}$ is composed of 
the angle and angular velocities of the two joints. 
The two-dimensional continuous action space 
\A\ was discretized using $3$ values per dimension 
(maximum positive, maximum negative or zero torque for each actuator), 
resulting in a total of $9$ discrete actions. 
We adopted a discount factor of $\gamma=0.9$. 

The reward received at each time step was $-\delta$, 
where $\delta$ is the Euclidean distance between 
the target position and the tip of the arm. 
The start state at each episode was defined as 
follows during training: the inner joint angle was sampled from an uniform distribution over $[0, 2\pi]$, the outer joint was sampled from an uniform distribution over $\{-\pi/2, \pi/2\}$, and both angular velocities were set to $0$ (during the evaluation phase two fixed start states were used---see below). We used a time step of $0.02$s and episodes lasted for $10$s ($500$ time steps). We defined $12$ target locations, $4$ of which we 
used for training and $8$ were reserved for testing (see Figure~\ref{fig:results_reacher}).

\subsubsection{Algorithms}
\label{sec:reacher_alg}

The baseline method used for comparisons 
in the reacher domain was the DQN algorithm by \ct{mnih2015human}. 
In order to make it possible for DQN to generalize across tasks we provided the target locations as part of the state description. The action-value function $\tilde{Q}$ was represented by a multi-layer perceptron (MLP) with two hidden layers of $256$ linear 
units followed by $\tanh$ non-linearities. The output of the network was a vector in $ \R^{9}$ with the estimated value associated with each action. The replay buffer adopted was large enough to retain all the transitions seen by the agent---that is, we never 
removed transitions from the buffer (this helps prevent DQN 
from ``forgetting'' previously seen tasks when learning new ones). 
Each value-function update used a mini-batch of $32$ transitions 
sampled uniformly from the 
replay buffer, and the associated minimization (line~\ref{it:ql_update} of 
Algorithm~\ref{alg:ql}) was carried out using the Adam optimizer
with a learning rate of $10^{-3}$~\cp{kingma2014adam}. 

SFDQN is similar to SFQL, whose pseudo-code is shown in Algorithm~\ref{alg:sfql}, with a few modifications to make learning with nonlinear function approximators more stable.
The vector of features $\vphi \in \R^{12}$ 
used by SFDQN was composed of the negation of the 
distances to all target 
locations.\footnote{In fact, instead of the negation of the distances $-\delta$ we  
used $1-\delta$ in the definition of both the rewards and the features 
$\vphi_i$. Since in our domain $\delta < 1$, this change made the rewards always 
positive. This helps preventing randomly-initialized value function 
approximations from dominating the `$\max$' operation in~(\ref{eq:pitmax}).}
We used a separate MLP to represent each of the four $\tvpsi_i$. 
The MLP architecture was the same as the one adopted with DQN,
except that in this case the output of the network was a matrix $\tmpsi_i \in \R^{9 \times 12}$ representing $\tvpsi_i(s,a) \in \R^{12}$ for each $a \in \A$. This means that the parameters $\Z_i$ in Algorithm~\ref{alg:sfql} should now be interpreted as weights of a nonlinear neural network. 
The final output of the network was computed 
as $\tvpsi_i^{_\t} \w_t$, where $\w_t \in \R^{12}$ is 
an one-hot vector indicating the task $t$ currently 
active. 
Analogously to the the target locations given as inputs to DQN, 
here we assume that $\w_t$ is provided by the environment. 
Again making the connection with Algorithm~\ref{alg:sfql}, this means that line~\ref{it:update_w} would be skipped.

Following~\ct{mnih2015human}, in order to make the training of the neural network more stable we updated $\tvpsi$ using a target network. This corresponds to replacing (\ref{eq:loss_psi}) with the following loss:
\begin{equation*}
\label{eq:loss_psi2}
\begin{array}{cl}
\mathrm{L}_Z(\tvpsi^{\pi_i}) \equiv  \mathrm{L}_Z(\Z^{\pi_i})  
= \E_{(s,a, s') \sim \mathcal{D}}\left[\left(\vphi(s,a,s')
+ \gamma \tvpsi_{^-}^{\pi_i}(s',a') - \tvpsi^{\pi_i}(s,a) \right)^2\right], 
\end{array}
\end{equation*}
where $\tvpsi_{^-}^{\pi_i}$ is a target network that remained fixed during updates and was periodically replaced by $\tvpsi^{\pi_i}$ at every $1000$ steps 
(the same configuration used for DQN). 
For each transition we updated all four MLPs $\tvpsi_i$
in order to minimize losses derived from~(\ref{eq:bellman_psi}). 
As explained in Section~\ref{sec:experiments},
the policies $\pi_i(s)$ used in~(\ref{eq:bellman_psi}) were the GPI policies associated 
with each training task, $\pi_i(s) \in \argmax_a \max_{j} \tvpsi_j(s,a)^\t \w_i$.
An easy way to modify Algorithm~\ref{alg:sfql}) to reflect 
this update strategy would be to 
select the action $a'$ in line~\ref{it:gpi_action_selection2} using $\pi_i(s)$
and then repeat the block in lines~\ref{it:update_previous_psi_b}--\ref{it:update_previous_psi_e}
for all $\w_i$, with $i \in \{1, 2, 3, 4\}$ and $i \ne t$, where 
$t$ is the current task.

\subsubsection{Experimental setup}

The agents were trained for $200\,000$ transitions 
on each of the $4$ training tasks. 
Data was collected using an $\epsilon$-greedy policy with 
$\epsilon = 0.1$ (for SFDQN, this corresponds to lines~\ref{it:epsilon_greedy_n} and~\ref{it:epsilon_greedy} of Algorithm~\ref{alg:sfql}). 
As is common in fixed episode-length control tasks, we excluded the terminal transitions during training to make the value of states independent of time, which corresponds to learning continuing policies~\cp{lillicrap2015continuous}.

During the entire learning process we monitored the performance of the agents on all $12$ tasks when using an $\epsilon$-greedy policy with $\epsilon = 0.03$. The results shown in Figure~\ref{fig:results_reacher} reflect the performance of this policy. Specifically, 
the return shown in the figure is the sum of rewards received 
by the $0.03$-greedy policy over two episodes starting from fixed states. Since the maximum possible return varies across tasks, we normalized the returns per task 
based on the performance of standard DQN on separate experiments 
(this is true for both training and test tasks). Specifically, we carried out the 
normalization as follows. First, we ran DQN $30$ times on each task and 
recorded the algorithm's performance before and after training. Let $\bar{G}_b$ and
$\bar{G}_a$ be the average performance of DQN over the $30$ runs 
before and after training, respectively. Then, if during our actual experiments 
DQN or SFDQN got a return of $G$, the normalized version of this metric was 
obtained as $G_n = (G - \bar{G}_b) / (\bar{G}_a - \bar{G}_b)$.
These are the values shown in Figure~\ref{fig:results_reacher}.
Visual inspection of videos 
extracted from the experiments with DQN alone 
suggests that the returns used for normalization were obtained by near-optimal policies that reach the targets almost directly. 

\section{Additional empirical analysis}
\label{sec:extra_results}

In this section we report empirical results that had to be left out of the main 
paper due to the space limit. Specifically, the objective of the experiments 
described here is to provide a deeper understanding of SFQL, in particular,
and of SFs, more generally. We use the four-room environment to carry out our
empirical investigation.

\subsection{Understanding the types of transfer promoted by SFs}

We start by asking why exactly SFQL performed so much better than QL and PRQL 
in our experiments (see Figure~\ref{fig:results_four_room}).
Note that there are several possible reasons for that to be the case. 
As shown in~(\ref{eq:sf}), SFQL uses a decoupled 
representation of the value function in which the environment's dynamics are
dissociated from the rewards. 
If the reward function of a given environment can be non-trivially decomposed 
in the form~(\ref{eq:reward}) assumed by SFQL, 
the algorithm can potentially build on this knowledge to quickly 
learn the value function and to adapt to 
changes in the environment, as discussed in Section~\ref{sec:sf}.
In our experiments not only did we know that such a non-trivial decomposition
exists, we actually provided such an information to SFQL---
either directly, through a handcrafted $\vphi$, or indirectly, 
by providing features that allow for \vphi\ to be recovered exactly. 
Although this fact should help explain the good results of SFQL,
it does not seem to be the main reason for the difference in 
performance. Observe in Figure~\ref{fig:results_four_room} how the advantage of 
SFQL over the other methods only starts to show up in the second task, and it 
only becomes apparent from the third task on. This suggests that 
the the algorithm's good performance is indeed due to some form of 
transfer.

But what kind of transfer, exactly? We note that SFQL naturally promotes two 
forms of transfer. The first one is of course a consequence of GPI.
As discussed in the paper, SFQL applies GPI by storing a set of SFs 
$\tvpsi^{\pi_i}$. Note though that SFs promote a weaker form of transfer even 
when only a \emph{single} $\tvpsi^{\pi}$ exists. To see why this is so, observe 
that if $\tvpsi^{\pi}$ persists from task $t$ to $t+1$, instead of arbitrary
approximations $\Qt^{\pi}$ one will have reasonable estimates of $\pi$'s value 
function under the current $\wt$. In other words, $\tvpsi^{\pi}$ will 
transfer knowledge about $\pi$ from one task to the other.

In order to have a better understanding of the two types of transfer promoted 
by SFs, we carried out experiments in which we tried to 
isolate as much as possible each one of them. 
Specifically, we repeated the 
experiments shown in Figure~\ref{fig:results_four_room} but now 
running SFQL with and without GPI (we can turn off GPI 
by replacing $c$ with $t$ in line~\ref{it:gpi_action_selection} 
of Algorithm~\ref{alg:sfql}).

\begin{figure}[tb]
\centering
 \includegraphics[scale=0.25]{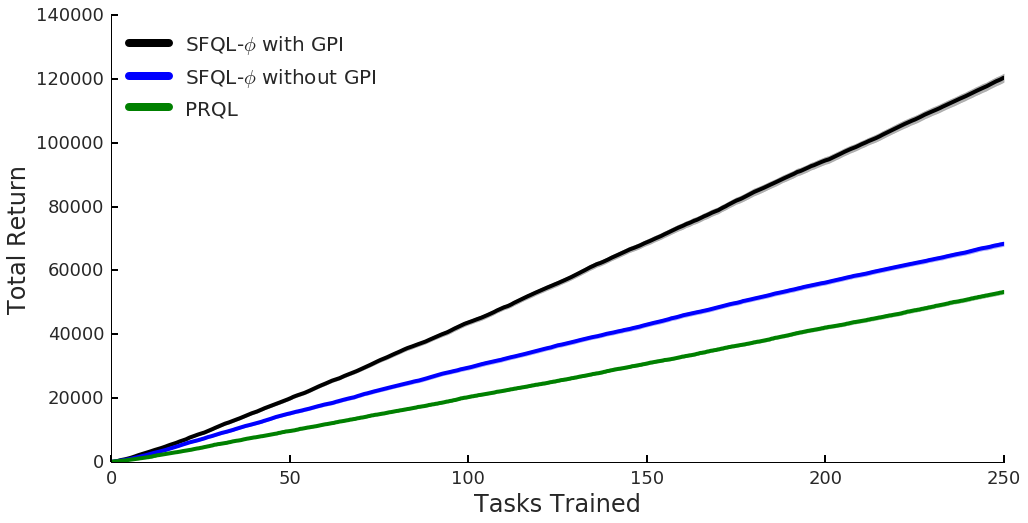}
 \caption{Understanding how much each type of transfer 
 promoted by SFs helps in performance.
 Results averaged over $30$ runs; 
 standard errors are shown as shadowed regions but are almost imperceptible
 at this scale. PRQL's results are shown for reference.
 \label{fig:sfql_transfer}}
\end{figure}

The results of our experiments are shown in Figure~\ref{fig:sfql_transfer}.
It is interesting to note that even without GPI SFQL initially outperforms PRQL. 
This is probably the combined effect of the two factors 
discussed above: the decoupled representation of the value function plus the 
weak transfer promoted by $\tvpsi^{\pi}$. 
Note though that, although these factors do give SFQL a head-start, 
eventually both algorithms reach the same performance level, as clear by 
the slope of the respective curves. 
In contrast, SFQL with GPI consistently outperforms the other two methods, 
which is evidence in favor of 
the hypothesis that GPI is indeed a crucial component of the proposed approach.
Another evidence in this direction is given in 
Figure~\ref{fig:sfql_transfer_functions}, which shows all the functions 
computed by the SFQL agent. Note how after only $200$ transitions into a 
new task SFQL already has a good approximation of the reward function, which,
combined with the set of previously computed $\tvpsi^{\pi_i}$, 
with $i < t$, provide a very informative value function even without the 
current $\tvpsi^{\pi_t}$.

\begin{figure}
\centering
\begin{tabular}{cccl}
\includegraphics[height=22mm,width=22mm]{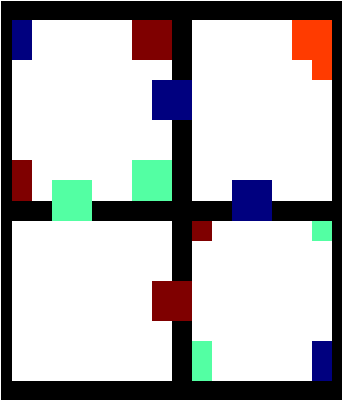} & 
\includegraphics[height=22mm,width=22mm]{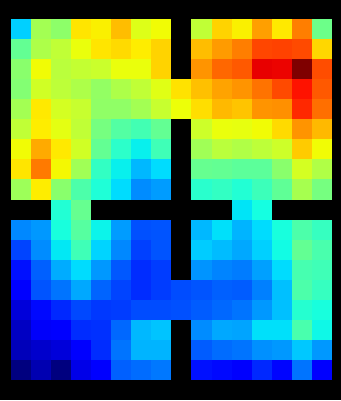} &
\includegraphics[height=22mm,width=22mm]{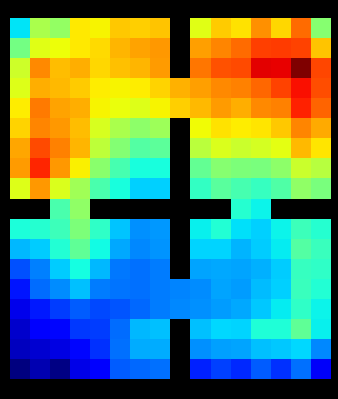} & \\
\includegraphics[height=22mm,width=22mm]{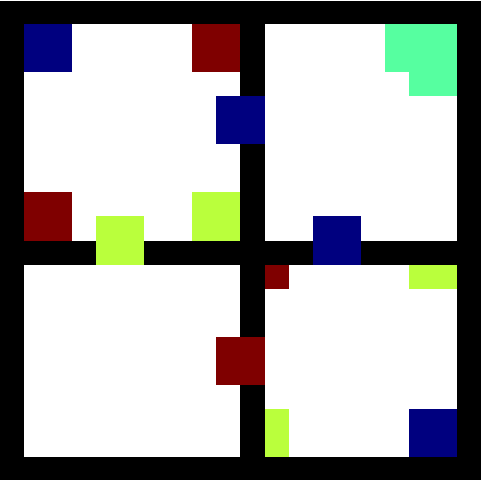} & 
\includegraphics[height=22mm,width=22mm]{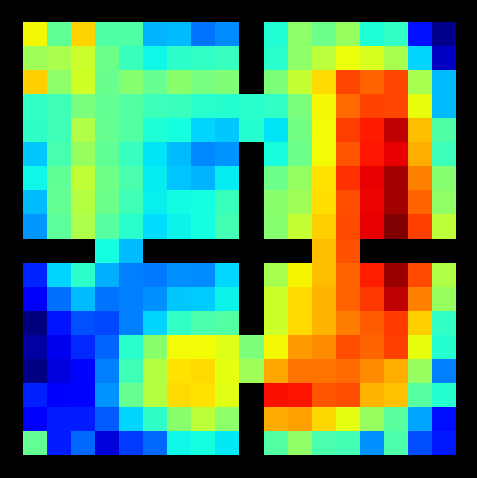} &
\includegraphics[height=22mm,width=22mm]{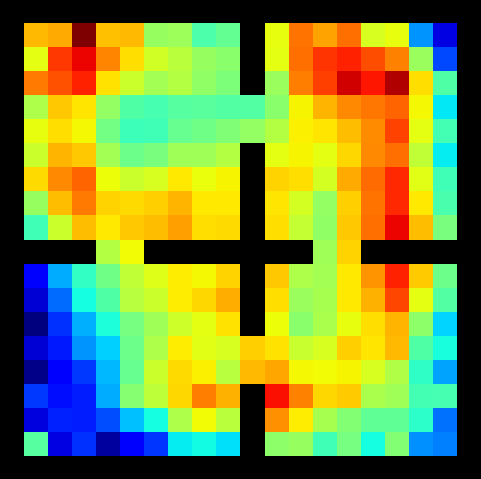} & \\ 
\includegraphics[height=22mm,width=22mm]{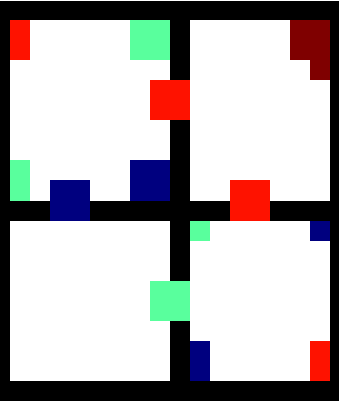} & 
\includegraphics[height=22mm,width=22mm]{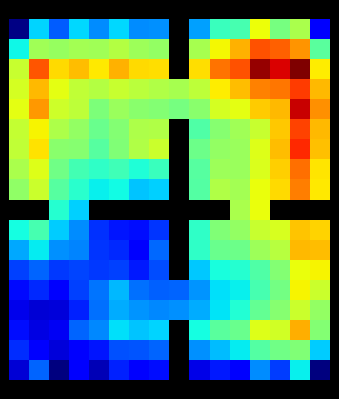} &
\includegraphics[height=22mm,width=22mm]{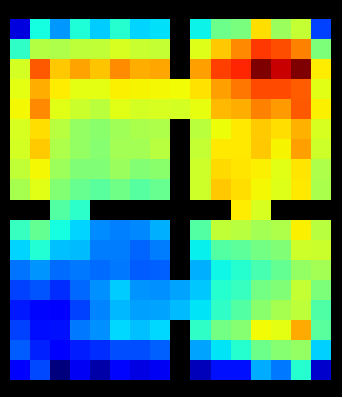} & \\ 
\\
$\tvphi(s,a,s')^\t \wt_t$ & 
$\max\limits_{a, i < t} \tvpsi_i(s,a)^\t \wt_t$ & 
$\max\limits_{a, i \le t} \tvpsi_i(s,a)^\t \wt_t$ 
\end{tabular}
\caption{Functions computed by SFQL after $200$ transitions into 
three randomly selected tasks (all objects present).
\label{fig:sfql_transfer_functions}}
\vspace{-4mm}
\end{figure}

\subsection{Analyzing the robustness of SFs}

As discussed in the previous section, part of the benefits provided by SFs come 
from the decoupled representation of the value function shown in~(\ref{eq:sf}), 
which depends crucially on a decomposition of the reward function 
$\tvphi(s,a,s')^\t\w \approx r(s,a,s')$. 
In Section~\ref{sec:experiments} we illustrated 
two ways in which such a decomposition can be obtained: by 
handcrafting \tvphi\ based on prior knowledge about the environment or 
by learning it from data. When \tvphi\ is learned it is obviously not 
reasonable to expect an exact decomposition of the reward function, but even 
when it is handcrafted the resulting reward model can be only an approximation 
(for example, the ``object detector'' used by the agents in the experiments 
with the four-room environment could be noisy). 
Regardless of the source of the imprecision, we do not want 
SFQL in particular, and more generally any method using SFs, to completely 
break with the addition of noise to \tvphi. In Section~\ref{sec:experiments} we 
saw how an approximate \tvphi\ can in fact lead to very good performance. In 
this section we provide a more systematic investigation of this matter by 
analyzing the performance of SFQL with \tvphi\ corrupted in different ways.

\begin{figure}[hbt]
\centering
\includegraphics[scale=1]{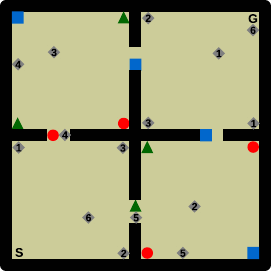}
\caption{Environment layout with additional objects.
Spurious objects are represented as diamonds; numbers indicate 
the classes of the objects.
\label{fig:forage_maze_spurious}}
\end{figure}

Our first experiment consisted in adding spurious features to $\vphi$ that do 
not help in any way in the approximation of the reward function. This reflects 
the scenario where the agent's set of predictions is a superset of the features 
needed to compute the reward (see in Section~\ref{sec:related_work} the 
connection with \ca{sutton2011horde}'s GVFs, \cy{sutton2011horde}). 
In order to implement this, we added objects to our environment that always 
lead to zero reward---that is, even though the SFs $\tvpsi^{\pi}$ learned by 
SFQL would predict the occurrence of these objects, such predictions would not 
help in the approximation of the reward function. Specifically, we added 
to our basic layout $15$ objects belonging to $6$ classes, as shown in 
Figure~\ref{fig:forage_maze_spurious},
which means that in this case SFQL used a $\tvphi \in \R^{10}$.
In addition to adding spurious features, we also corrupted  \tvphi\ by adding 
to each of its components, in each step, a sample from a normal distribution
with mean zero and different standard deviations $\sigma$.

The outcome of our experiment is shown in Figure~\ref{fig:sfql_noise}.
Overall, the results are as expected, with both spurious features and noise 
hurting the performance of SFQL. Interestingly, the two types of 
perturbations do not seem to interact very strongly, since their 
effects seem to combine in an additive way. More important, SFQL's performance 
seems to degrade gracefully as a result of either intervention, which 
corroborate our previous experiments showing that the proposed 
approach is robust to approximation errors in \tvphi.

\begin{figure}[t]
\centering
 \includegraphics[width=0.8\textwidth]{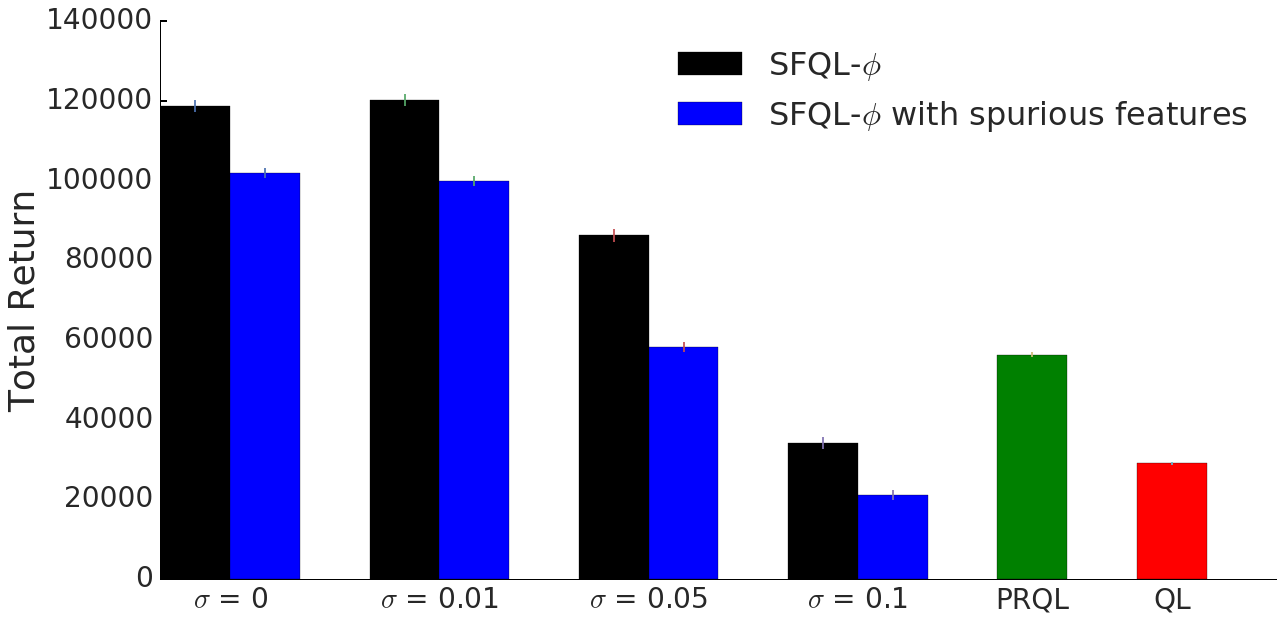}
 \caption{Analyzing SFQL's robustness to distortions in $\tvphi$.
 PRQL's and QL's results are shown for reference. Results averaged over $30$ 
runs;  standard errors are shown on top of each bar.
\label{fig:sfql_noise}}
\end{figure}

\makeatletter
\apptocmd{\thebibliography}{\global\c@NAT@ctr 29\relax}{}{}
\makeatother

\end{document}